\documentclass[letterpaper]{article} 
\usepackage{aaai2026}  
\usepackage{times}  
\usepackage{helvet}  
\usepackage{courier}  
\usepackage[hyphens]{url}  
\usepackage{graphicx} 
\usepackage{natbib}  
\usepackage{caption} 
\usepackage{algorithm}
\usepackage{algorithmic}

\setcitestyle{maxcitenames=1, mincitenames=3} 

\usepackage{amsmath,amsfonts,bm}

\def\eqref#1{equation~\ref{#1}}

\def\1{\bm{1}}

\DeclareMathAlphabet{\mathsfit}{\encodingdefault}{\sfdefault}{m}{sl}
\SetMathAlphabet{\mathsfit}{bold}{\encodingdefault}{\sfdefault}{bx}{n}

\newcommand{\R}{\mathbb{R}}

\DeclareMathOperator*{\argmax}{arg\,max}

\usepackage{url}
\usepackage{tikz} 
\usepackage{amsmath,amssymb,amsthm,textcomp}
\usepackage{multicol}
\usepackage{comment}
\usepackage{stackengine}
\usepackage{booktabs}
\usepackage{microtype}      
\usepackage{capt-of}
\usepackage{mathtools}

\usepackage{stackengine}
\usepackage{pifont}
\usepackage{booktabs}
\usepackage{graphics}

\usepackage{tikz}

\theoremstyle{plain}
\newtheorem{theorem}{Theorem}[section]
\newtheorem{proposition}[theorem]{Proposition}

\theoremstyle{definition}

\theoremstyle{remark}

\newlength{\RoundedBoxWidth}
\newsavebox{\GrayRoundedBox}
\newenvironment{GrayBox}[1][\dimexpr\columnwidth-1.9ex]%
{\setlength{\RoundedBoxWidth}{\dimexpr#1}
\begin{lrbox}{\GrayRoundedBox}
\begin{minipage}{\RoundedBoxWidth}}%
{   \end{minipage}
\end{lrbox}
\begin{center}
\begin{tikzpicture}%
\draw node[draw=black,fill=black!4,rounded corners,%
inner sep=1ex,text width=\RoundedBoxWidth]%
{\usebox{\GrayRoundedBox}};
\end{tikzpicture}
\end{center}}

\usepackage{newfloat}
\usepackage{listings}
\DeclareCaptionStyle{ruled}{labelfont=normalfont,labelsep=colon,strut=off} 
\floatstyle{ruled}
\newfloat{listing}{tb}{lst}{}
\floatname{listing}{Listing}
\title{Principled Analysis of Deep Reinforcement Learning Evaluation and Design Paradigms}

\author{
    Ezgi Korkmaz
}
\affiliations{
}

\usepackage{bibentry}

\begin{document}

\maketitle

\begin{abstract}
Starting from the utilization of deep neural networks to approximate the state-action value function that led to winning one of the most challenging games, to algorithmic advancements that allowed solving problems without even explicitly stating the rules of the challenge at hand, reinforcement learning research has been the center of remarkable scientific progress for the past decade. In this paper, we focus on the key ingredients of this research progress and we analyze the canonical evaluation and design paradigms in reinforcement learning. 
We introduce the theoretical foundations of scaling laws in reinforcement learning and show that the asymptotic performance of reinforcement learning algorithms does not have a monotone relationship between performance rankings and data-regimes.
We conduct large-scale experiments and our results demonstrate that a line of reinforcement learning research under the canonical design and evaluation paradigms resulted in incorrect conclusions.
Our analysis and results provide a core analysis on scaling, capacity and complexity of deep reinforcement learning.
\end{abstract}

\section{Introduction}

Founded on rigorous theoretical guarantees,
reinforcement learning research achieved high acceleration upon the proposal of the initial study on approximating the state-action value function via deep neural networks \citep{mn15,nisan20,julian20, lee24, korkmaz25}. 
A line of highly successful deep reinforcement learning algorithms have been proposed \citep{hado16, wang16, hado18, hessel21, kaptur23,korkmaz24icml} 
from focusing on different architectural ideas 
to foundations targeting overestimation, 
all of which were designed and tested in the high-data regime, i.e. two hundred million frame training.
An alternative recent line of research with an extensive amount of publications focused on pushing the performance bounds of deep reinforcement learning policies in the low-data regime, i.e. with one hundred thousand environment interaction training. 
Many different concepts in current reinforcement learning research, 
from architectural proposals
to learning underlying dynamics of the environment, 
experienced accelerated progress and significant attention, growing into several major research fields,
solely based on policy performance comparisons demonstrated in the low-data regime benchmark.

In this paper, we focus on evaluation paradigms, implicit assumptions and canonical methodological choices made 
in deep reinforcement learning research and demonstrate that there is a significant overlooked underlying premise driving this line of research without being explicitly discussed: that the performance profiles of deep reinforcement learning algorithms have a monotonic relationship with different sample-complexity regimes. 
We show that this implicit assumption, that is commonly shared amongst a large collection of low-data regime studies, 
shapes how the canonical design and evaluation choices are made in deep reinforcement learning research and represents a prominent misdirection in scientific progress.
The suboptimal conclusions obtained from these canonical choices shape future research directions with incorrect reasoning. 
We show that these methodological decisions fuel incorrect justifications and conclusions, thereby misdirecting research efforts toward certain concepts for several years.
Thus, in our paper we target these underlying premises and aim to answer the following questions:
\vskip-0.5in
\begin{center}
\emph{What are the implicit assumptions and canonical choices in deep reinforcement learning research that fundamentally affect the conclusions made?} \\
\vskip0.04in
\emph{What is the foundational relationship between sample complexity and the algorithmic performance from the data-scarce regime to the asymptotic regime?}\\
\end{center}

Hence, to be able to answer the questions raised above, in our paper we focus on underlying design and evaluation paradigms in deep reinforcement learning and make the following contributions:
\begin{itemize}
\item We analyze the evaluation paradigms and canonical methodological choices in deep reinforcement learning research, and introduce the theoretical foundations on how these methodological choices affect algorithm design, performance comparisons and algorithmic conclusions. 
Our analysis lays the foundations on scaling, capacity and complexity of deep reinforcement learning.
\item Our theoretical analysis proves that the performance profile has a non-monotonic relationship with the asymptotic sample complexity and the low-data sample complexity regime.
Regarding the central focus of the large scale implicit assumption instances, our results reveal that the canonical methodological choices made in a line of deep reinforcement learning research have led to incorrect justifications and conclusions.
\item We conduct large scale extensive experiments for a comprehensive and a diverse portfolio of deep reinforcement learning baseline algorithms in both the low-data regime and the high-data regime Arcade Learning Environment benchmark. Our results demonstrate that recent algorithms proposed and evaluated in the Arcade Learning Environment 100K benchmark are significantly affected by the implicit assumption on the relationship between performance profiles and sample complexity resulting in systematic bias in algorithmic evaluation.
\end{itemize}
\section{Background and Preliminaries}

The reinforcement learning problem is formalized as a Markov Decision Process (MDP) represented as a tuple $\langle S, A, \mathcal{P}, \mathcal{R}, \gamma, \rho_0 \rangle$ where $S$ represents the state space, $A$ represents the set of actions, $\mathcal{P}$ represents the transition probability distribution on $S\times A \times S$, $\mathcal{R}:S \times A \to \R$ represents the reward function, and $\gamma \in (0,1]$ represents the discount factor. 
The aim in reinforcement learning is to learn an optimal policy $\pi(s,a)$ that maps state observations to actions
$\pi:S\to \Delta(A)$, 
which maximizes the expected cumulative discounted rewards $R = \mathbb{E}_{a_t \sim \pi(s_t,\cdot)} \sum_t \gamma^t \mathcal{R}(s_t,a_t,s_{t+1})$.
This objective is achieved by constructing a state-action value function that learns for each state-action pair the expected cumulative discounted rewards that will be obtained if action $a \in A$ is executed in state $s \in S$.
\vskip-0.18in
\[
 \mathcal{Q}(s,a) = \sum_{s'}  \mathcal{P}(s'|s,a) [\mathcal{R}(s,a, s') + \gamma \mathcal{V}(s')]
 \]
\vskip-0.08in
In settings where the state space and/or action space is large enough that the state-action value function $\mathcal{Q}(s,a)$ cannot be held in a tabular form, a function approximator is used. Thus, for deep reinforcement learning the $\mathcal{Q}$-function is approximated via deep neural networks 
\vskip-0.2in
\begin{align*}
\theta_{t+1} = \theta_{t} + \alpha (&\mathcal{R}(s_t,a_t, s_{t+1}) \\
&+ \gamma \mathcal{Q}(s_{t+1}, \argmax_a \mathcal{Q}(s_{t+1},a;\theta_t);\theta_t)
\end{align*}
\vskip-0.17in
\begin{align*}
 \qquad \qquad\qquad\qquad\qquad  - \mathcal{Q}(s_t,a_t;\theta_t)) \nabla_{\theta_t} \mathcal{Q}(s_t,a_t;\theta_t).
\end{align*}
\vskip-0.03in
\noindent  \textbf{Dueling Architecture:} 
The dueling architecture \citep{wang16} outputs two streams of fully connected layers for both estimating the advantage $\mathcal{A}(s,a)$ for each action in a given state $s$, $\mathcal{A}(s,a) = \mathcal{Q}(s,a) - \max_a \mathcal{Q}(s,a)$ and the state values $\mathcal{V}(s)$.
In particular, the last layer of the dueling architecture contains the forward mapping
$
\mathcal{Q}(s,a;\theta, \alpha, \beta) = \mathcal{V}(s; \theta, \beta) + \big( \mathcal{A}(s,a;\theta, \alpha) - \max_{a' \in A} \mathcal{A}(s,a';\theta, \alpha) \big)$
where $\theta$ represents the parameters of the convolutional layers and $\alpha$ and $\beta$ represent the parameters of the fully connected layers outputting the advantage and state value estimates respectively.

\noindent \textbf{Inherent High-Capacity Models:} The initial algorithm that has been proposed to have inherent high-capacity is C51. In particular, the projected Bellman update for the $i^{\textrm{th}}$ atom is computed as
\vspace{-0.27cm}
\begin{align}
(\Phi \mathcal{T} \mathcal{Z}_\theta(s_t,a_t))_i &= \sum_{j}^{\mathcal{N}-1} \big[ 1- \dfrac{|[\mathcal{T}z_j]^{v_{\textrm{max}}}_{v_{\textrm{min}}}   -z_i|}{\Delta z} \big]^1_0 \nonumber \\
& \qquad \qquad \quad  \tau_j(s_{t+1}, \max_{a \in A} \mathbb{E} \mathcal{Z}_{\theta}(s_{t+1},a)) \nonumber
\end{align}
\vskip-0.09in
where $z_i = v_{\textrm{min}} + i\Delta z : 0 \leq i < \mathcal{N}$ represents the set of atoms in categorical learning, and the atom probabilities are learnt as a parametric model \citep{bell17}
\vskip-0.3in
\begin{equation}
\tau_i(s_t, \max_{a \in A} \mathbb{E} \mathcal{Z}_{\theta}(s_t,a)) = \dfrac{e^{\theta_i (s_t,a_t)}}{\sum_j e^{\theta_j (s_t,a_t)}} \:\: \textrm{,} \:\: \Delta z := \dfrac{v_\textrm{max} - v_\textrm{min}}{\mathcal{N}-1} \nonumber
\end{equation}
Following this baseline the $\mathcal{Q}$RD$\mathcal{Q}$N algorithm \cite{will18} is proposed to learn the quantile projection
\vskip-0.17in
\begin{align}
\mathcal{T}\mathcal{Z}(s_t,a_t) = \mathcal{R}(&s_t,a_t,s_{t+1}) \nonumber \\
&+ \gamma \mathcal{Z}(s_{t+1},\argmax_{a \in A} \mathbb{E}_{z \sim \mathcal{Z}(s_{t+1},a_{t+1})} [z]) \nonumber
\end{align}
\vskip-0.08in
with $s_{t+1} \sim \mathcal{P}(\cdot|s_t,a_t)$ where $\mathcal{Z} \in Z$ represents the quantile distribution of an arbitrary value function. Following this study the I$\mathcal{Q}$N algorithm \cite{dab18} is proposed to learn the full quantile function instead of learning a discrete set of quantiles as in the $\mathcal{Q}$RD$\mathcal{Q}$N algorithm. The I$\mathcal{Q}$N algorithm objective is to minimize the loss function
\vskip-0.2in
\begin{align}
\mathcal{L} =\dfrac{1}{\mathcal{K}} &\sum_{i=1}^\mathcal{K} \sum_{j=1}^\mathcal{K'} \rho_{\delta} (\mathcal{R}(s_t,a_t,s_{t+1}) \\
&+ \gamma \mathcal{Z}_{{\delta'_j}}(s_{t+1}, \argmax_{a \in A} \mathcal{Q}_\beta(s_t,a_t)) - \mathcal{Z}_{\delta_i}(s_t,a_t)) \nonumber 
\end{align}
\vskip-0.07in
where $\rho_{\delta}$ represents the Huber quantile regression loss, and $\mathcal{Q}_\beta = \int^1_0 \mathcal{F}^{-1}_\mathcal{Z}(\delta) d\beta(\delta)$. Note that $\mathcal{Z}_\delta =\mathcal{F}^{-1}_\mathcal{Z}(\delta)$ is the quantile function of the random variable $\mathcal{Z}$ at $\delta \in [0,1]$.

\section{Low-data Regime versus Asymptotic Performance}
\label{theorysec}
Our paper discovers both with extensive 
empirical analysis and theoretical investigation 
that asymptotic performance of reinforcement learning algorithms does not necessarily provide any information nor indication on their relative performance ranking in the low-data regime. The results provided in Section \ref{experiments} extensively demonstrate that a large body of work in reinforcement learning research carried this assumption and resulted in incorrect conclusions.
In this section, we introduce the foundational basis for our discovery revealed by our extensive empirical analysis in Section \ref{experiments}
in optimization of non-stationary policies, i.e. rewards and transitions that can vary with each step in an episode, in undiscounted, finite-horizon MDPs with linear function approximation.
In particular, a finite horizon MDP is represented as a tuple $\langle S, A, \mathcal{P}, \mathcal{R}, \mathcal{H} \rangle$ where $S$ is the set of states, and $A$ represents the set of actions. For each time step $t\in [\mathcal{H}]=\{1,\dots,\mathcal{H}\}$, state $s$, and action $a$ the transition probability kernel $\mathcal{P}_t(s'|s,a)$ gives the probability distribution over the next state, and the reward $\mathcal{R}_t(s,a,s')$ gives the immediate rewards. A non-stationary policy
$\pi = (\pi_1,\dots,\pi_\mathcal{H})$ induces a state-action value function given by
\vskip-0.2in
\begin{align*}
\mathcal{Q}_t^\pi(s,a) = \mathbb{E}
\left[\sum_{h=t}^\mathcal{H} 
 \mathcal{R}_h(s_h,\pi_h(s_h), s_{h+1}) \bigg| s_h=s, a_h=a \right] \\
\end{align*}
\vskip-0.24in
where we let $a_h \sim \pi_h(s_h)$,
 and the corresponding value function $\mathcal{V}_t^\pi (s) = \mathcal{Q}_t(s,\pi_t(s))$.
The optimal non-stationary policy $\pi^*$ has value function $\mathcal{V}_t^*(s) = \mathcal{V}_t^{\pi^{*}}(s)$ satisfying $\mathcal{V}_t^*(s) = \sup_{\pi} \mathcal{V}_t^{\pi}(s).$
The objective is to learn a sequence of non-stationary policies $\pi^{k}$ for $k\in\{1,\dots,\mathcal{K}\}$ while interacting with an unknown MDP in order to minimize the regret, which is measured asymptotically over $\mathcal{K}$ episodes of length $\mathcal{H}$, 
$\textsc{Regret}(\mathcal{K}) = \sum_{k=1}^\mathcal{K} \left(\mathcal{V}_1^*(s^k_1) - \mathcal{V}_1^{\pi^{k}}(s^k_1)\right)$
where $s^k_1 \in S$ is the starting state of the $k$-th episode.
Regret sums up the gap between the expected rewards obtained by the sequence of learned policies $\pi^{k}$ and those obtained by $\pi^*$ when learning for $\mathcal{K}$ episodes.
In the linear function approximation setting there is a feature map $\phi_t:S\times A \to \mathbb{R}^{d_t}$ for each $t \in [\mathcal{H}]$ that sends a state-action pair $(s,a)$ to the $d_t$-dimensional vector $\phi_t(s,a)$. Then, the state-action value function $\mathcal{Q}_t(s,a)$ is parameterized by a vector $\theta_t\in \mathbb{R}^{d_t}$ so that $\mathcal{Q}_t(\theta_t)(s,a) = \phi_t(s,a)^{\top}\theta_t$.
Recent theoretical work in this setting gives an algorithm along with a lower bound that matches the regret achieved by the algorithm up to logarithmic factors.
\begin{theorem}[\cite{ZanetteLKB20}]
\label{thm:optregretalg}
Under appropriate normalization assumptions there is an algorithm that learns a sequence of policies $\pi^{k}$ achieving regret $\textsc{Regret}(\mathcal{K}) = \tilde{O}\left(\sum_{t=1}^\mathcal{H} d_t\sqrt{\mathcal{K}} + \sum_{t=1}^\mathcal{H} \sqrt{d_t}\mathcal{I}\mathcal{K}\right)$, 
where $\mathcal{I}$ is the inherent Bellman error. Furthermore, this regret bound is optimal for this setting up to logarithmic factors in $d_t,\mathcal{K}$ and $\mathcal{H}$ whenever $\mathcal{K}=\Omega((\sum_{t=1}^\mathcal{H} d_t)^2)$, in the sense that for any level of inherent Bellman error $\mathcal{I}$ and sequence of feature dimensions $\{d_t\}_{t=1}^H$, there exists a class of MDPs $\mathcal{C}(\mathcal{I},\{d_t\}_{t=1}^H)$ where any algorithm achieves at least as much regret on at least one MDP in the class.
\end{theorem}
The class of MDPs  $\mathcal{C}(\mathcal{I},\{d_t\}_{t=1}^H)$ constructed in Theorem \ref{thm:optregretalg} additionally satisfies the following properties. First, every MDP in $\cup_{\mathcal{I},\{d_t\}_{t=1}^H}\mathcal{C}(\mathcal{I},\{d_t\}_{t=1}^H)$ has the same transitions (up to renaming of states and actions).
Second, for each fixed value of the inherent Bellman error $\mathcal{I}$ and the dimensions $\{d_t\}_{t=1}^H$, every MDP in $\mathcal{C}(\mathcal{I},\{d_t\}_{t=1}^H)$ utilizes the same feature map $\phi_t(s_t,a_t)$.
Thus one can view the class $\mathcal{C}(\mathcal{I},\{d_t\}_{t=1}^H)$ as encoding one "underlying" true environment defined by the transitions, with varying values of $\mathcal{I}$ and $\{d_t\}_{t=1}^H$ corresponding to varying levels of function approximation accuracy, and model capacity for the underlying environment. For simplicity of notation we will focus on the setting where $d_t = d$ for all $t\in\{1,\dots H\}$ and write $\mathcal{C}(\mathcal{I},d)$ for the class of MDPs constructed in Theorem \ref{thm:optregretalg} for this setting.
Utilizing this point of view, we can then prove the following theorem on the relationship between the performance in the asymptotic and low-data regimes.

\newcommand{\dlow}{d_{\beta}}
\newcommand{\dhigh}{d_{\alpha}}
\newcommand{\Ilow}{\mathcal{I}_{\beta}}
\newcommand{\Ihigh}{\mathcal{I}_{\alpha}}
\newcommand{\Klow}{\mathcal{K}_{\text{low}}}
\newcommand{\Khigh}{\mathcal{K}_{\text{high}}}

\begin{theorem}[\emph{Non-monotonicity Across Regimes}]
\label{prop:lowdatavsasymp}
For any $\epsilon > 0$, let $\dhigh$ be any feature dimension, and let $\dlow = \dhigh^{1-\epsilon/2}$.
Then there exist thresholds $\Klow < \Khigh$ and inherent Bellman error levels $\Ilow > \Ihigh$ such that
\begin{enumerate}
    \item There is an algorithm achieving regret $\textsc{Regret}_{\text{low}}(\mathcal{K})$ when $\mathcal{K} < \Klow$ for all MDPs in $\mathcal{C}(\Ilow,\dlow)$.
    However, every algorithm has regret at least $\widetilde{\Omega}\left(\dlow^{\epsilon/2}\textsc{Regret}_{\text{low}}(\mathcal{K})\right)$ when $\mathcal{K} < \Klow$ on some MDP $M\in\mathcal{C}(\Ihigh,\dhigh)$.

    \item There is an algorithm achieving regret $\textsc{Regret}_{\text{high}}(\mathcal{K})$ when $\mathcal{K} > \Khigh$ for all MDPs in $\mathcal{C}(\Ihigh,\dhigh)$.
    However, every algorithm has regret at least $\widetilde{\Omega}\left(\dhigh^{\epsilon}\textsc{Regret}_{\text{high}}(\mathcal{K})\right)$ on some MDP $M\in\mathcal{C}(\Ilow,\dlow)$ when  $\mathcal{K} > \Khigh$.
\end{enumerate}
\end{theorem}
\begin{proof}
Let $\epsilon > 0$ and consider
$\dlow = \dhigh^{1-\frac{\epsilon}{2}}, \Ilow = \frac{1}{\dhigh^{\epsilon}\sqrt{\dlow}}, \Ihigh = \frac{1}{\dhigh^{\frac{1}{2}+2\epsilon}}, \Klow = \dhigh^{2 + \epsilon}, \Khigh = \dhigh^{2 + 4\epsilon}$
We begin with the proof of part 1.
Therefore, for $\mathcal{K} < \Klow$, $\sqrt{\dlow}\Ilow\mathcal{K} = \dhigh^{-\epsilon}\mathcal{K} < \dhigh^{1-\frac{\epsilon}{2}}\sqrt{\mathcal{K}} = \dlow\sqrt{\mathcal{K}}$.
    Therefore, by Theorem \ref{thm:optregretalg} there exists an algorithm achieving regret
\begin{align*}
    \textsc{Regret}_{\text{low}}(\mathcal{K}) &= \tilde{O}\left(\mathcal{H} \dlow\sqrt{\mathcal{K}} + \mathcal{H} \sqrt{\dlow}\Ilow\mathcal{K}\right) \\
    &= \widetilde{O}\left(\dlow\sqrt{\mathcal{K}}\right)
\end{align*}
in every MDP $M\in \mathcal{C}(\Ilow,\dlow)$.
Further, since $\Klow = \dhigh^{2+\epsilon} > \widetilde{\Omega}\left(\dhigh^2\right)$, the lower bound from Theorem \ref{thm:optregretalg} applies to the class of MDPs $\mathcal{C}(\Ihigh,\dhigh)$ for all $\mathcal{K} \in \left[\widetilde{\Omega}\left(\dhigh^2\right),\Klow\right]$.
In particular, every algorithm receives regret at least
\begin{align*}
    \textsc{Regret}(\mathcal{K})
    &= \widetilde{\Omega}\left(\mathcal{H} \dhigh\sqrt{\mathcal{K}} + \mathcal{H} \sqrt{\dhigh}\Ihigh\mathcal{K}\right)\\
    &> \widetilde{\Omega}\left(\mathcal{H} \dlow^{\frac{1}{1-\epsilon/2}}\sqrt{\mathcal{K}}\right)  \\
    &> \widetilde{\Omega}\left(\mathcal{H}\dlow^{\frac{\epsilon/2}{1-\epsilon/2}} \dlow\sqrt{\mathcal{K}}\right)
\end{align*}
Thus, $\textsc{Regret}(\mathcal{K}) > \widetilde{\Omega}\left(\dlow^{\epsilon/2}\textsc{Regret}_{\text{low}}(\mathcal{K})\right)$. For part 2, note that for $\mathcal{K} > \Khigh$ we have both $\sqrt{\dhigh}\Ihigh\mathcal{K} = \dhigh^{-2\epsilon}\mathcal{K} > \dhigh^{-2\epsilon}\sqrt{\mathcal{K}\cdot\Khigh} > \dhigh\sqrt{\mathcal{K}}$ and $\sqrt{\dlow}\Ilow\mathcal{K} > \dhigh^{-\epsilon}\sqrt{\mathcal{K}\cdot\Klow} = \dhigh^{1+\epsilon}\sqrt{\mathcal{K}} > \dlow\sqrt{\mathcal{K}}$.
\begin{figure*}[t]
\footnotesize
\begin{center}
\hskip-0.14in
\stackunder[6pt]{\includegraphics[scale=0.198]{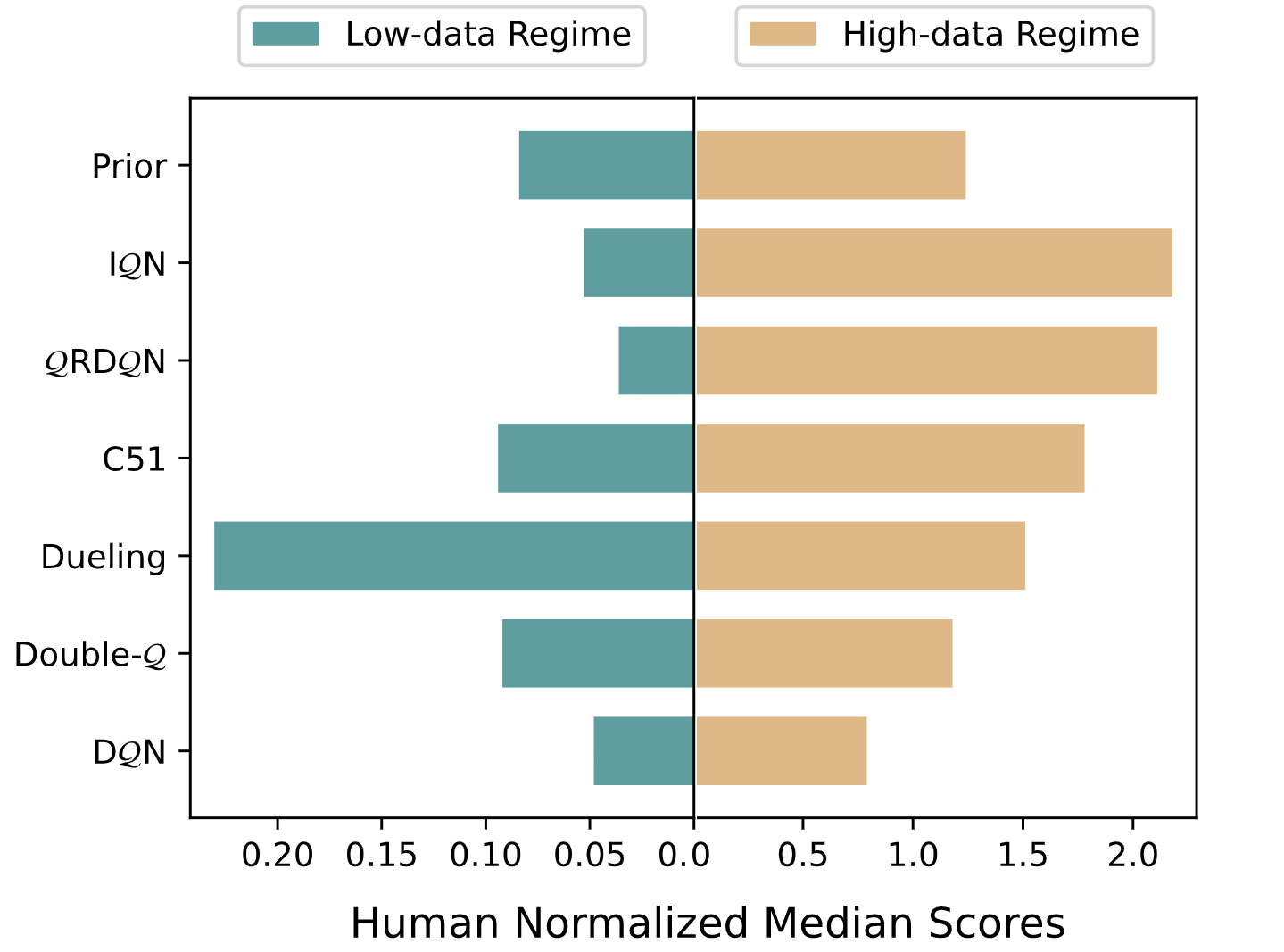}}{Asymptotic vs Low-Data} 
\hskip-0.12in
\stackunder[4pt]{\includegraphics[scale=0.174]{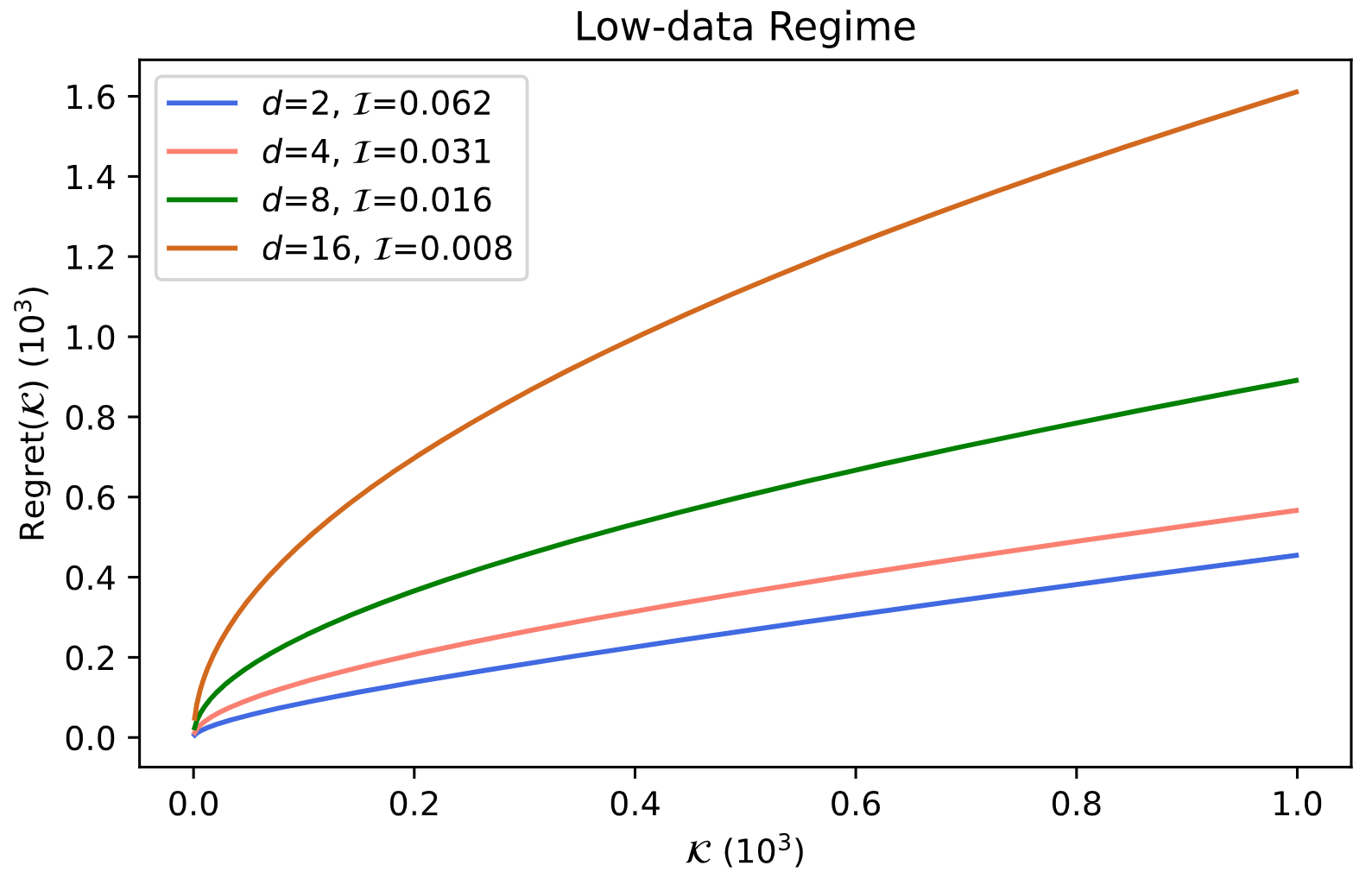}}{Regret in the Low-data Regime}
\hskip-0.05in
\stackunder[4pt]{\includegraphics[scale=0.18]{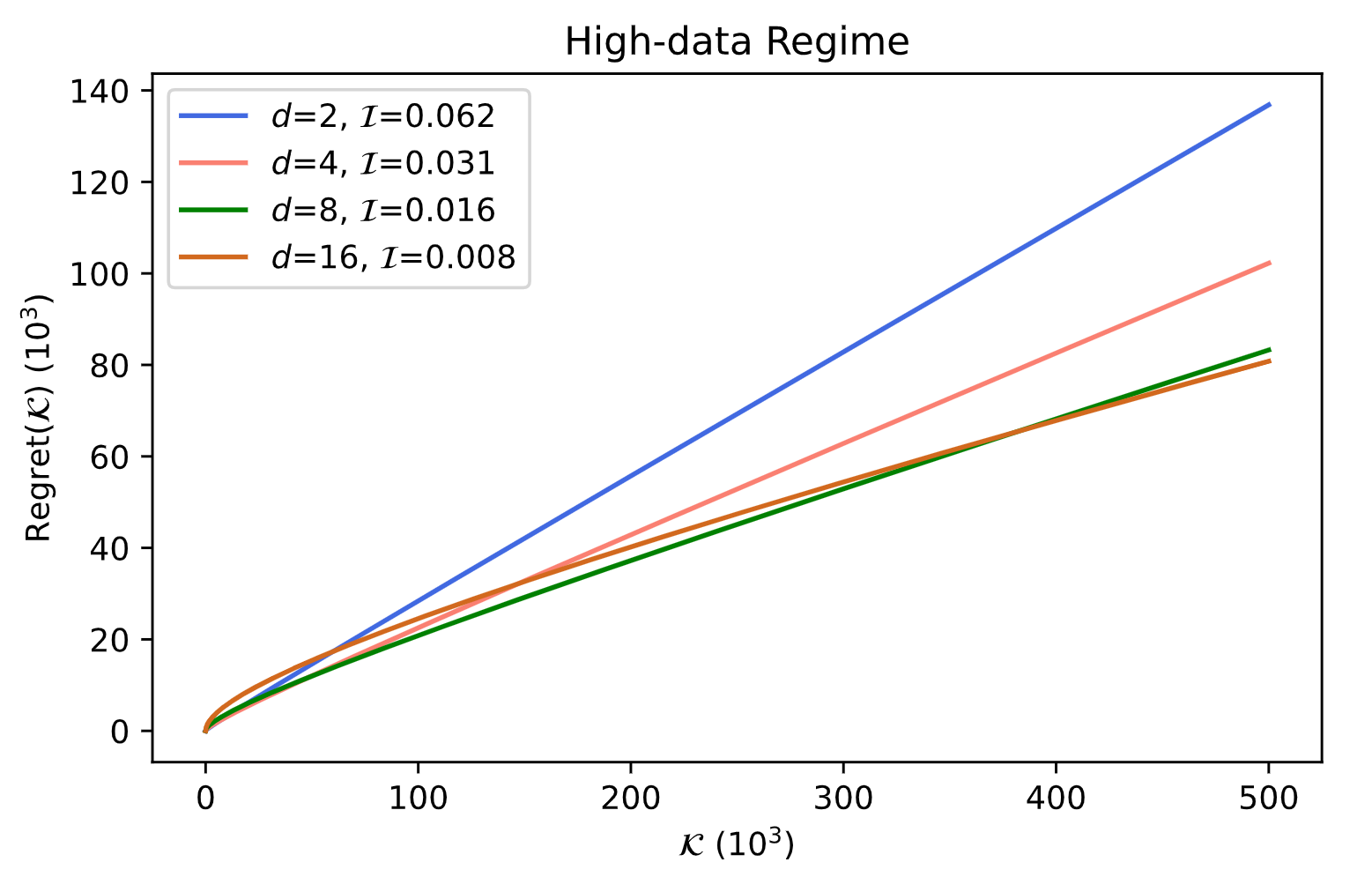}}{Regret in the High-data Regime}
\hskip 0.1pt
\end{center}
\vskip -0.14in
\caption{Left: Scaling laws of reinforcement learning: Baseline comparison of algorithms that were proposed and developed in the high-data regime in the Arcade Learning Environment in both high-data regime and low-data regime. Middle: Regret in the low-data regime. Right: Regret in the high-data regime. }
\label{data}
\end{figure*}
Therefore by Theorem \ref{thm:optregretalg} that for $\mathcal{K} > \Khigh$ there exists an algorithm achieving regret
\begin{align*}
    \textsc{Regret}_{\text{high}}(\mathcal{K})
    &= \tilde{O}\left(\mathcal{H} \dhigh\sqrt{\mathcal{K}} + \mathcal{H} \sqrt{\dhigh}\Ihigh\mathcal{K}\right)  \\
    &= \tilde{O}\left(\mathcal{H} \sqrt{\dhigh}\Ihigh\mathcal{K}\right).
\end{align*}
for every MDP $M\in\mathcal{C}(\Ihigh,\dhigh)$. However, by the lower bound in Theorem \ref{thm:optregretalg}, for $\mathcal{K} > \Khigh$ every algorithm receives regret at least
\begin{align*}
    \textsc{Regret}(\mathcal{K}) &= \widetilde{\Omega}\left(\mathcal{H} \dlow\sqrt{\mathcal{K}} + \mathcal{H} \sqrt{\dlow}\Ilow\mathcal{K}\right) \\
    &> \widetilde{\Omega}\left(\mathcal{H}\sqrt{\dlow}\Ilow\mathcal{K}\right) 
    = \widetilde{\Omega}\left(\mathcal{H}\dhigh^{-\epsilon}\mathcal{K}\right)  \\
    &= \widetilde{\Omega}\left(\dhigh^{\epsilon}\mathcal{H}\dhigh^{-2\epsilon}\mathcal{K}\right)  
    = \widetilde{\Omega}\left(\dhigh^{\epsilon}\mathcal{H}\sqrt{\dhigh}\Ihigh\mathcal{K}\right) \\
    &> \widetilde{\Omega}\left( \dhigh^{\epsilon}\textsc{Regret}_{\text{high}}(\mathcal{K})\right)
\end{align*}
\vskip-0.2in
\end{proof}
Theorem \ref{prop:lowdatavsasymp} introduces the provable trade-off between performance in the low-data regime, i.e. $\mathcal{K}<\Klow$, and the high-data regime, i.e. $\mathcal{K} > \Khigh$.
In particular, in the low-data regime lower capacity function approximation, i.e. lower feature dimension $\dlow$, with larger approximation error, i.e. larger inherent Bellman error $\Ilow$, can provably outperform larger capacity models, i.e. feature dimension $\dhigh$, with smaller approximation error, i.e. inherent Bellman error $\Ihigh$.
Furthermore, the relative performance is reversed in the high-data regime $\mathcal{K} > \Khigh$.
Thus, asymptotic performance of an algorithm is neither indicative nor carries any relevant information on the expected performance of the algorithm when training data is scarce (i.e. limited).

\section{The Assumption of Monotonicity and Performance Rankings}
\label{samplecomplexity}

The instances of the implicit assumption that the performance profile of an algorithm in the high-data regime will translate to the low-data regime monotonically appear in almost all of the studies conducted in the low-data regime. 
In particular, we see that when this line of work was being conducted the best performing algorithm in the high-data regime was an inherently high capacity model, i.e. based on learning the state action value distribution. 
Hence, there are many cases in the literature (e.g. DRQ, OTR, DER, CURL, SimPLE, Efficient-Zero) where all the newly proposed algorithms in the low-data regime are being compared to an algorithm 
that inherently produces a higher capacity model
\textbf{under the implicit assumption} that an algorithm that is state-of-the-art in the high-data regime must be the state-of-the-art in the low-data regime. 
The large scale experiments provided in Section \ref{experiments} demonstrate the impact of this implicit assumption and provide a guideline for a principled analysis and evaluation.
In particular, the results reported in Section \ref{experiments} prove that the performance profile of an algorithm in the high-data regime does not monotonically transfer to the low-data regime.
Due to this extensive focus throughout the literature on low-data regime comparisons to algorithms
that inherently learn higher capacity models,
we provide additional theoretical analysis for the empirically observed sample complexity results in the low to high-data regime in deep reinforcement learning.
The following proposition demonstrates a precise justification of these issues: whenever there are two different actions where the true mean state-action values are within $\epsilon$, an approximation error of $\epsilon$ in total variation distance $d_{TV}$ for 
$\mathcal{D}(s,a)$
of one of the actions can be sufficient to reverse the order of the means.

\begin{proposition}[\emph{Sufficiency of error of $\epsilon$ in total variation distance}]
\label{prop:tvisgood}
Fix a state $s$ and consider two actions $a,\hat{a}$. 
Let $\mathcal{D}(s,a)$ be the true state-action value distribution of $(s,a)$, and let $\mathcal{Z}(s,a) \sim \mathcal{D}(s,a)$. 
Suppose that $\mathbb{E}[\mathcal{Z}(s,a)] = \mathbb{E}[\mathcal{Z}(s,\hat{a})] + \epsilon$.
Then there is a random variable $\mathcal{Y}$ such that $d_{TV}(\mathcal{Y},\mathcal{Z}(s,a)) \leq \epsilon \: \textrm{ and } \: \mathbb{E}[\mathcal{Z}(s,\hat{a})] \geq \mathbb{E}[\mathcal{Y}].$
\end{proposition}
The proof is provided in the supplementary material. 
Proposition \ref{prop:tvisgood} shows that to have the correct ranking of the actions the state-action value distribution must be learnt with error at most $\epsilon$. Standard results on sample complexity for discrete distributions then imply that algorithms that learn the state-action value distribution with fixed support size $k$, i.e C51, require $k/\epsilon^2$ samples to achieve total variation distance at most $\epsilon$.
More advanced algorithms such as $\mathcal{Q}$RD$\mathcal{Q}$N and I$\mathcal{Q}$N do away with the assumption that the support is known. 
This allows a more flexible representation in order to more accurately represent
state-action values, but, as we will show, leads to a further increase in the sample complexity.
The $\mathcal{Q}$RD$\mathcal{Q}$N algorithm models it 
as a uniform mixture of $\mathcal{N}$ Dirac deltas on the reals i.e.
$\mathcal{Z}(s,a) = \frac{1}{\mathcal{N}}\sum_{i=1}^\mathcal{N} \delta_{\theta_i(s,a)}$,
where $\theta_i(s,a) \in \mathbb{R}$ is a parametric model.

\begin{proposition}[\emph{Sample Complexity with Unknown Support}]
\label{prop:qrdqncomplexity}
Let $\mathcal{N} > \mathcal{M} \geq 2$, $\epsilon > \frac{\mathcal{M}}{4\mathcal{N}}$, and $\theta_i\in \mathbb{R}$ for $i \in [\mathcal{N}]$. The number of samples required to learn a model of the form
$\mathcal{Z} = \frac{1}{\mathcal{N}}\sum_{i=1}^\mathcal{N} \delta_{\theta_i}$ to within total variation distance $\epsilon$ is $\Omega\left(\frac{\mathcal{M}}{\epsilon^2}\right)$.
\end{proposition}
The proof is provided in the supplementary material. 
Note that the lower bound in Proposition \ref{prop:qrdqncomplexity} can be significantly larger than $k/\epsilon^2$ samples.

\begin{figure*}[t]
\footnotesize
\begin{center}
\vskip -0.18in
\stackunder[3pt]{\includegraphics[scale=0.2]{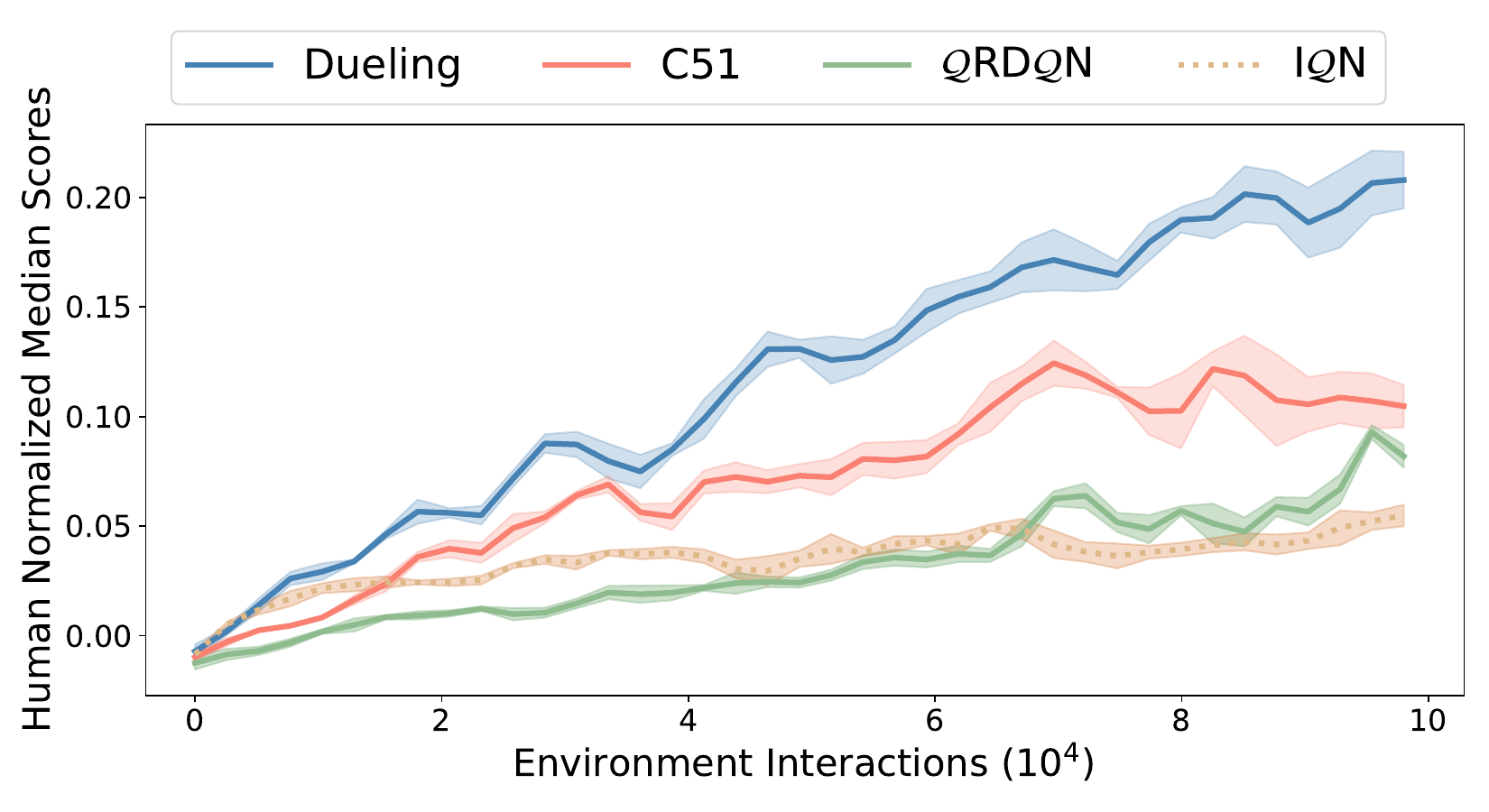}}{\scriptsize{Median}}
\stackunder[3pt]{\includegraphics[scale=0.2]{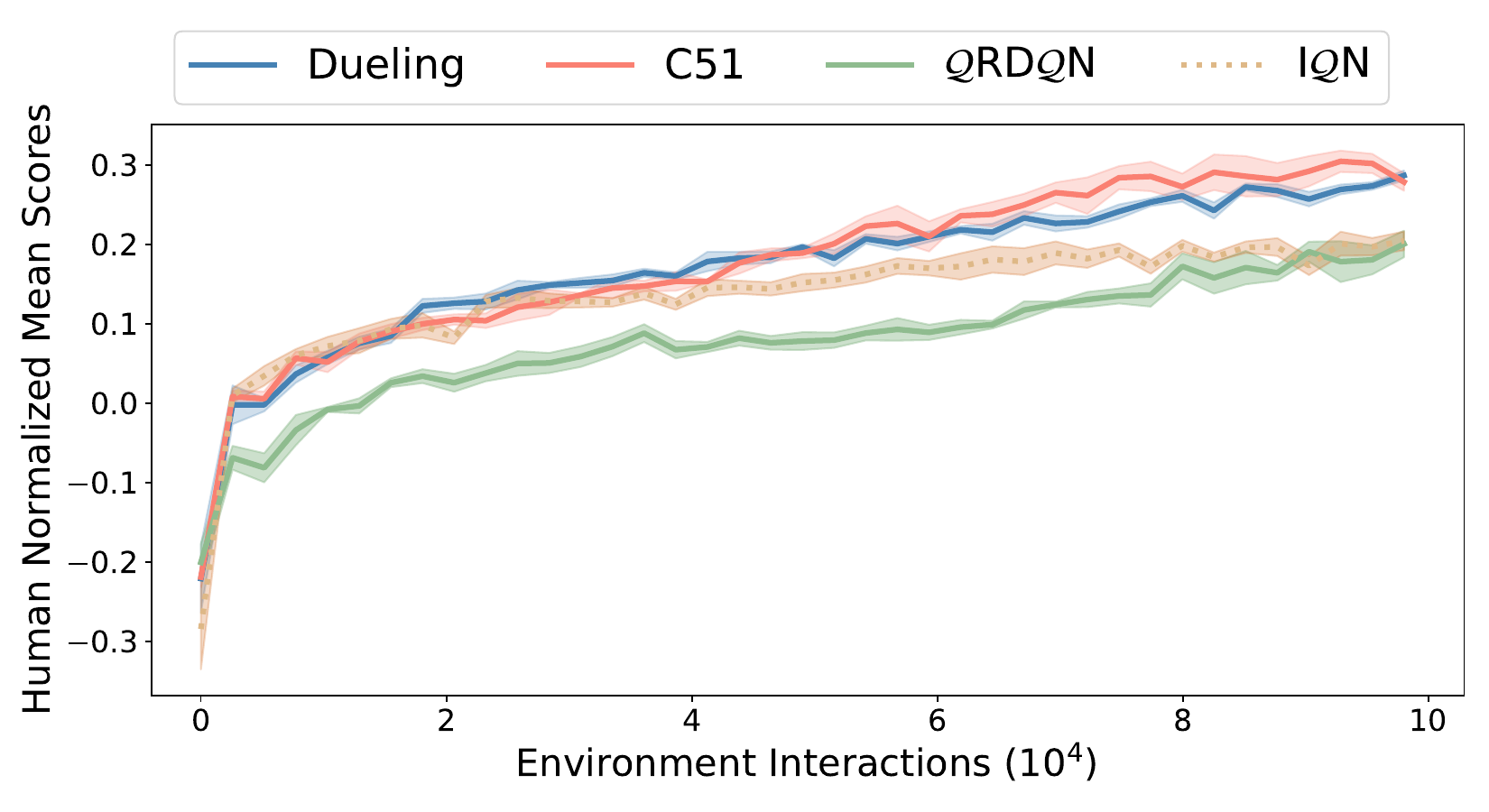}}{\scriptsize{Mean}}
\stackunder[3pt]{\includegraphics[scale=0.2]{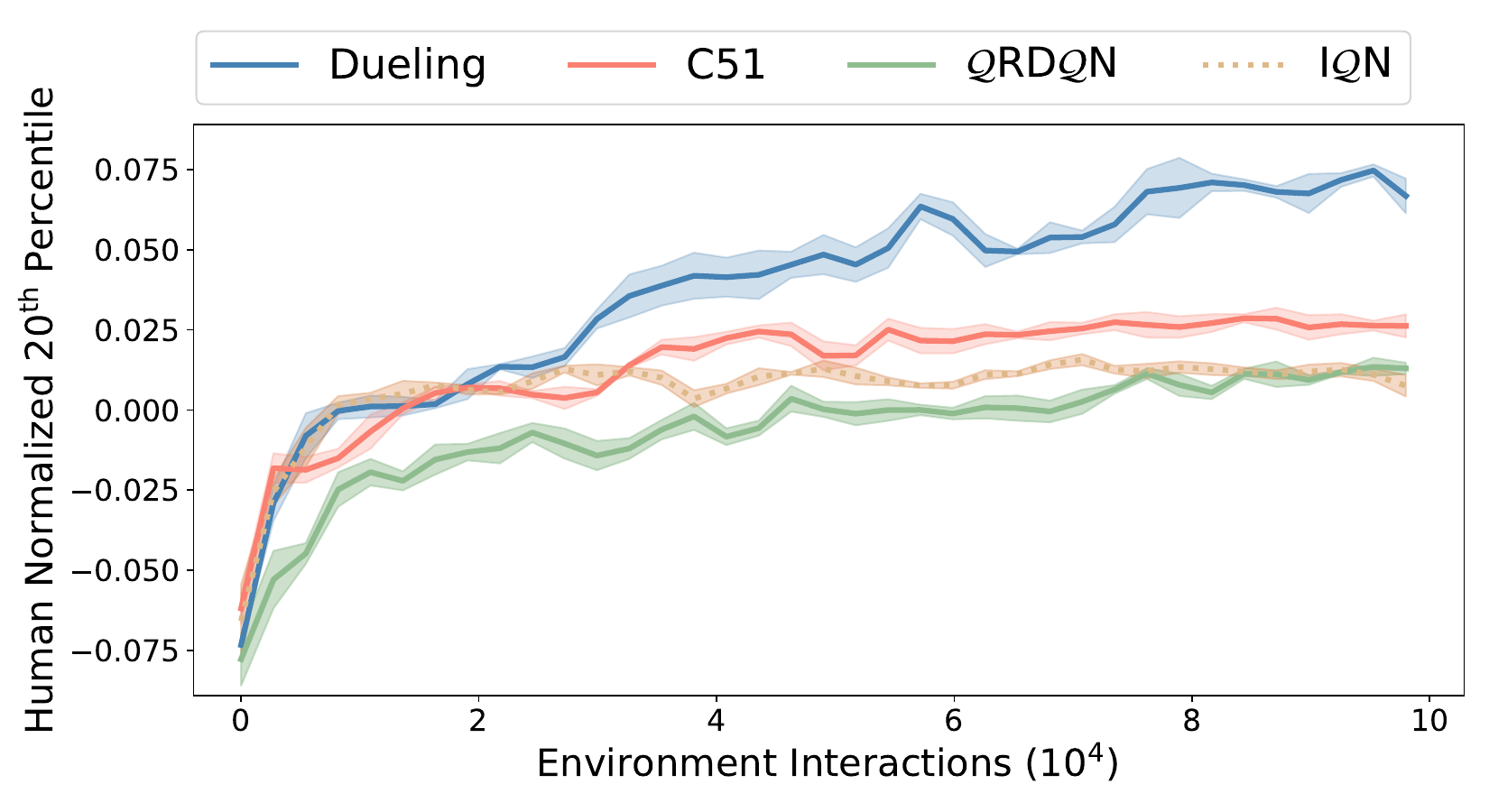}}{\scriptsize{20$^{\textrm{th}}$ Percentile}}\\
\stackunder[3pt]{\includegraphics[scale=0.2]{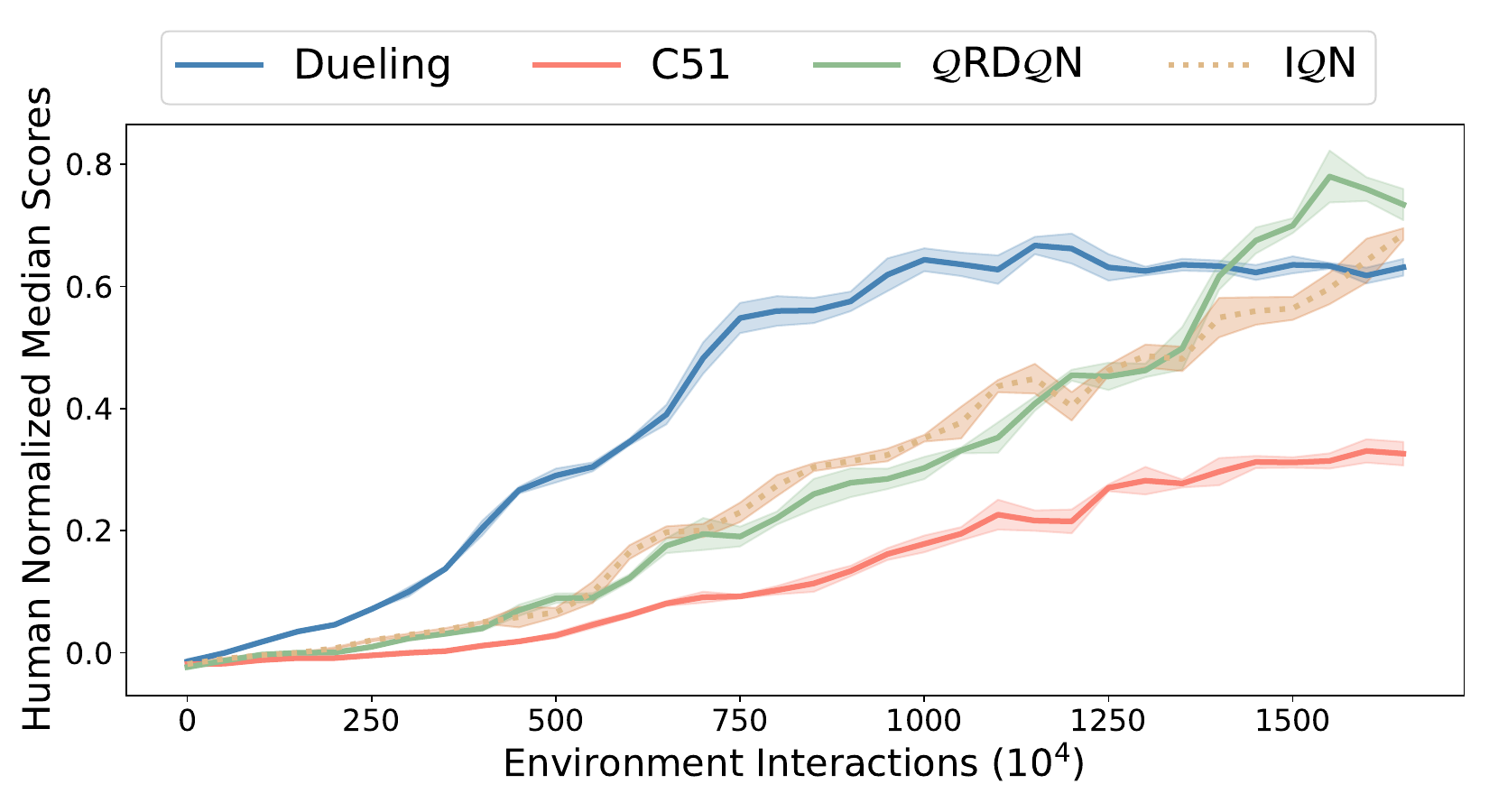}}{\scriptsize{Median}}
\stackunder[3pt]{\includegraphics[scale=0.2]{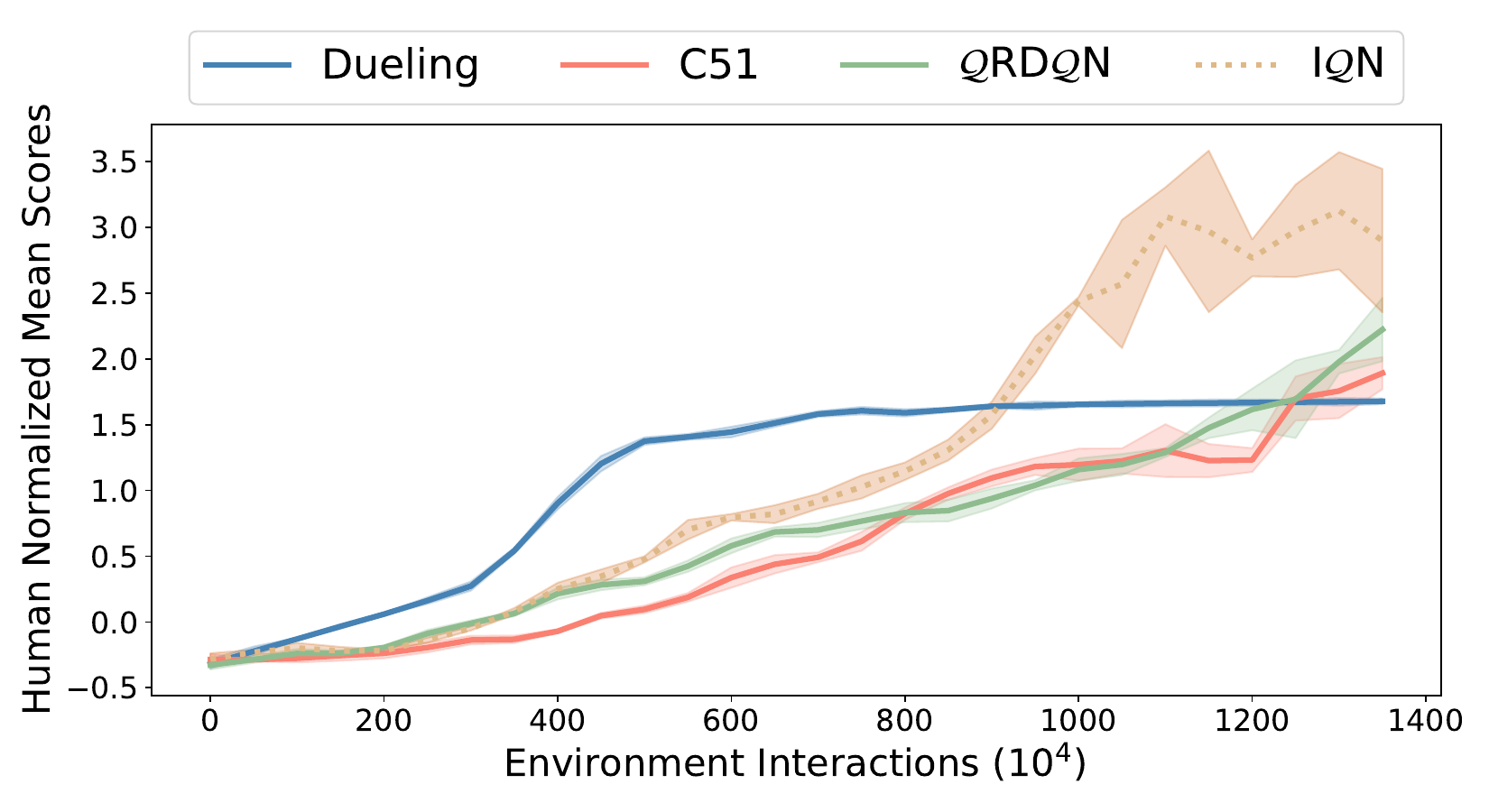}}{\scriptsize{Mean}}
\stackunder[3pt]{\includegraphics[scale=0.2]{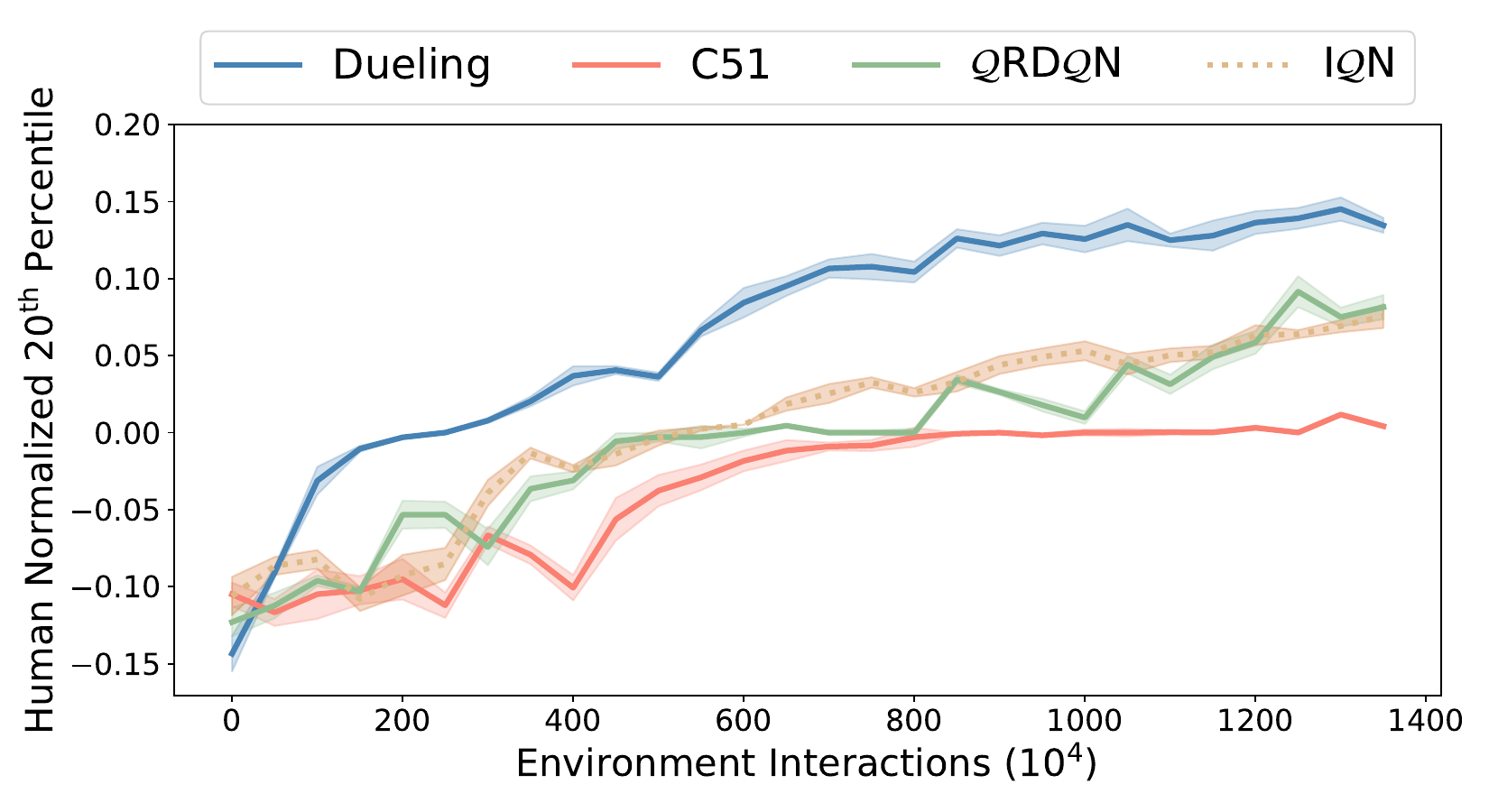}}{\scriptsize{20$^{\textrm{th}}$ Percentile}}
\end{center}
\vskip -0.2in
\caption{Up: Human normalized median, mean and 20$^\textrm{th}$ percentile results for the dueling algorithm, C51, I$\mathcal{Q}$N and $\mathcal{Q}$RD$\mathcal{Q}$N in the Arcade Learning Environment 100K benchmark. Down: Human normalized median, mean, and 20$^\textrm{th}$ percentile results for the dueling algorithm, C51, I$\mathcal{Q}$N and $\mathcal{Q}$RD$\mathcal{Q}$N in the high-data regime towards 200 million frame.}
\label{all}
\vskip -0.14in
\end{figure*}
\begin{table*}[t!]
\vskip -0.04in
\centering
\caption{Large scale comparison of $\mathcal{Q}$-based deep reinforcement learning algorithms with human normalized mean, median and 20$^\textrm{th}$ percentile results in the Arcade Learning Environment 100K benchmark for D$\mathcal{Q}$N \citep{mn15}, deep Double-$\mathcal{Q}$ \citep{hado16}, dueling \citep{wang16}, Prior \citep{tom16}, C51, $\mathcal{Q}$RD$\mathcal{Q}$N and I$\mathcal{Q}$N \citep{dab18}.}
\vskip -0.25in
\begin{tabular}{lccccr}
\toprule
Algorithms                      & Human Normalized Median    & Human Normalized Mean      &  20$^\textrm{th}$ Percentile  \\
\midrule
D$\mathcal{Q}$N                & 0.0481$\pm$0.0036          & 0.1535$\pm$0.0119         &  0.0031$\pm$0.0032              \\
Double-$\mathcal{Q}$           & 0.0920$\pm$0.0181          & \textbf{0.3169$\pm$0.0196}&  0.0341$\pm$0.0042              \\
Dueling                        & \textbf{0.2304$\pm$0.0061} & 0.2923$\pm$0.0060         &  \textbf{0.0764$\pm$0.0037}    \\
C51                            & 0.0941$\pm$0.0081          & 0.3106$\pm$0.0199         &   0.0274$\pm$0.0024                  \\
$\mathcal{Q}$RD$\mathcal{Q}$N  & 0.0820$\pm$0.0037          & 0.2171$\pm$0.0098         &   0.0189$\pm$0.0031                  \\
I$\mathcal{Q}$N                & 0.0528$\pm$0.0058          & 0.2050$\pm$0.0123         &   0.0091$\pm$0.0011                  \\
Prior                          & 0.0840$\pm$0.0018          & 0.2792$\pm$0.0123         &   0.0267$\pm$0.0042                  \\
\bottomrule
\end{tabular}
\label{comparison}
\vskip -0.14in
\end{table*}

\section{Principled Evaluation Framework}

In Section \ref{experiments}, we systematically explain and discuss the underlying design paradigms, the implicit assumptions and the methodological choices made in deep reinforcement learning research that led to incorrect conclusions. In this section we introduce the principled evaluation framework to ensure the research progress we obtain in deep reinforcement learning is reliable and scientifically robust. 
\vspace{-0.2cm}
\begin{GrayBox}
\vskip-0.02in
\noindent \textbf{I. Assumptions matter:} Performance rankings across regimes are non-monotone.\\\hskip-0.1in
\noindent \textbf{II. Biases in Evaluation:} Including algorithms in the comparison benchmark based on the monotonicity assumption will create biased evaluation.\\
\noindent \textbf{III. Core Algorithms:} Core algorithms must be included in the comparison benchmarks.\\
\noindent \textbf{IV. Inherent Capacity:} Inherent capacity and dimensionality will provide insights on performance rankings across regimes.\\
\noindent \textbf{V. Biases in Datasets:} Creating datasets based on the monotonicity assumption will create biased benchmarks.
\vskip-0.1in
\end{GrayBox}
\vskip-0.04in

\section{Large Scale Empirical Analysis}
\label{experiments}
The empirical analysis is conducted in the Arcade Learning Environment (ALE) \citep{mn15}. 
The Double $\mathcal{Q}$-learning algorithm is trained via \citet{hado16} initially proposed by \citet{hasselt10}.
The dueling algorithm is trained via \citet{wang16}. The prior algorithm refers to the prioritized experience replay algorithm proposed by \citet{tom16}.
The experiments are run  
with Haiku as the neural network library, Optax \citep{optax} as the optimization library, and RLax for the reinforcement learning library \citep{jax2}. 
All of the results are reported with the standard error of the mean.
For the full list of algorithms, details on the hyperparameters, direct references and the detailed explanations of the baselines please see the supplementary material.
To provide a complete picture of the sample complexity we conducted our experiments in both low-data, i.e. the Arcade Learning Environment 100K benchmark, and high data regime, i.e. baseline 200 million frame training. 
Note that human normalized score is computed as follows: $\textrm{Score}_{\textrm{HN}} = (\textrm{Score}_{\textit{agent}} - \textrm{Score}_{\textit{random}})/(\textrm{Score}_{\textit{human}} - \textrm{Score}_{\textit{random}})$.

\begin{figure*}[t]
    \footnotesize
    \begin{center}
    \stackunder[1pt]{\includegraphics[scale=0.12]{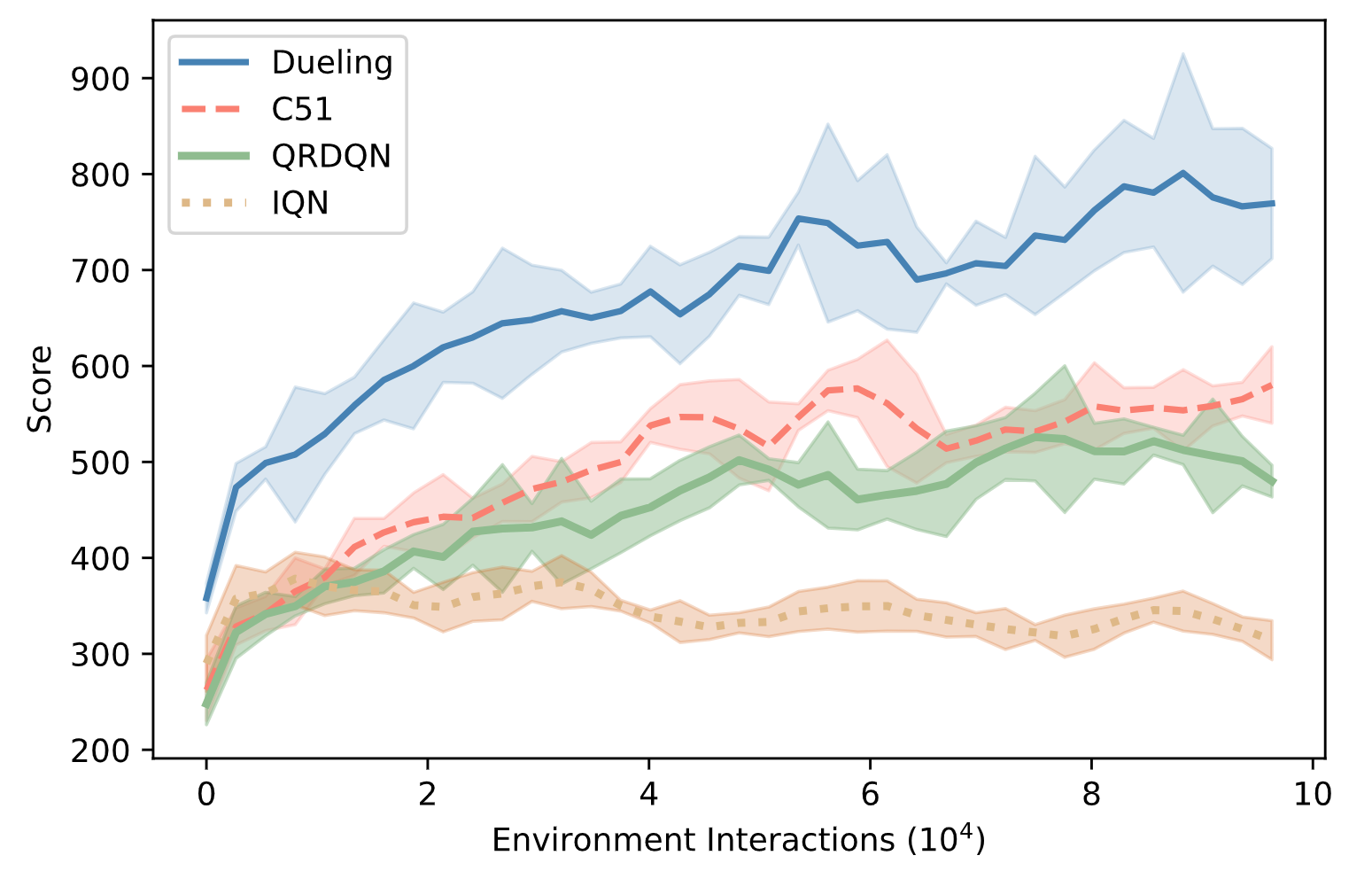}}{\scriptsize{Alien}}
    \stackunder[1pt]{\includegraphics[scale=0.125]{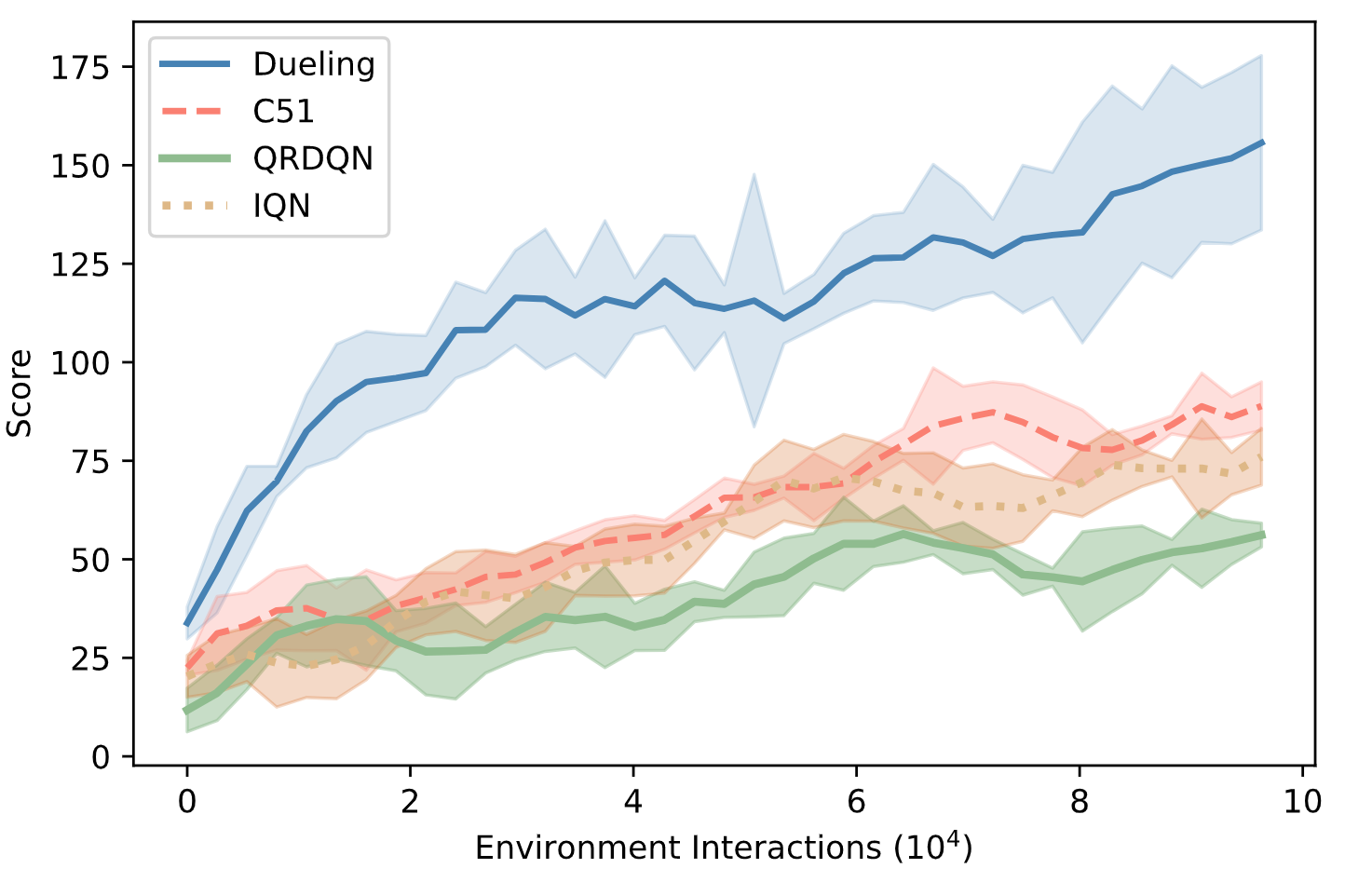}}{\scriptsize{Amidar}}
    \stackunder[1pt]{\includegraphics[scale=0.12]{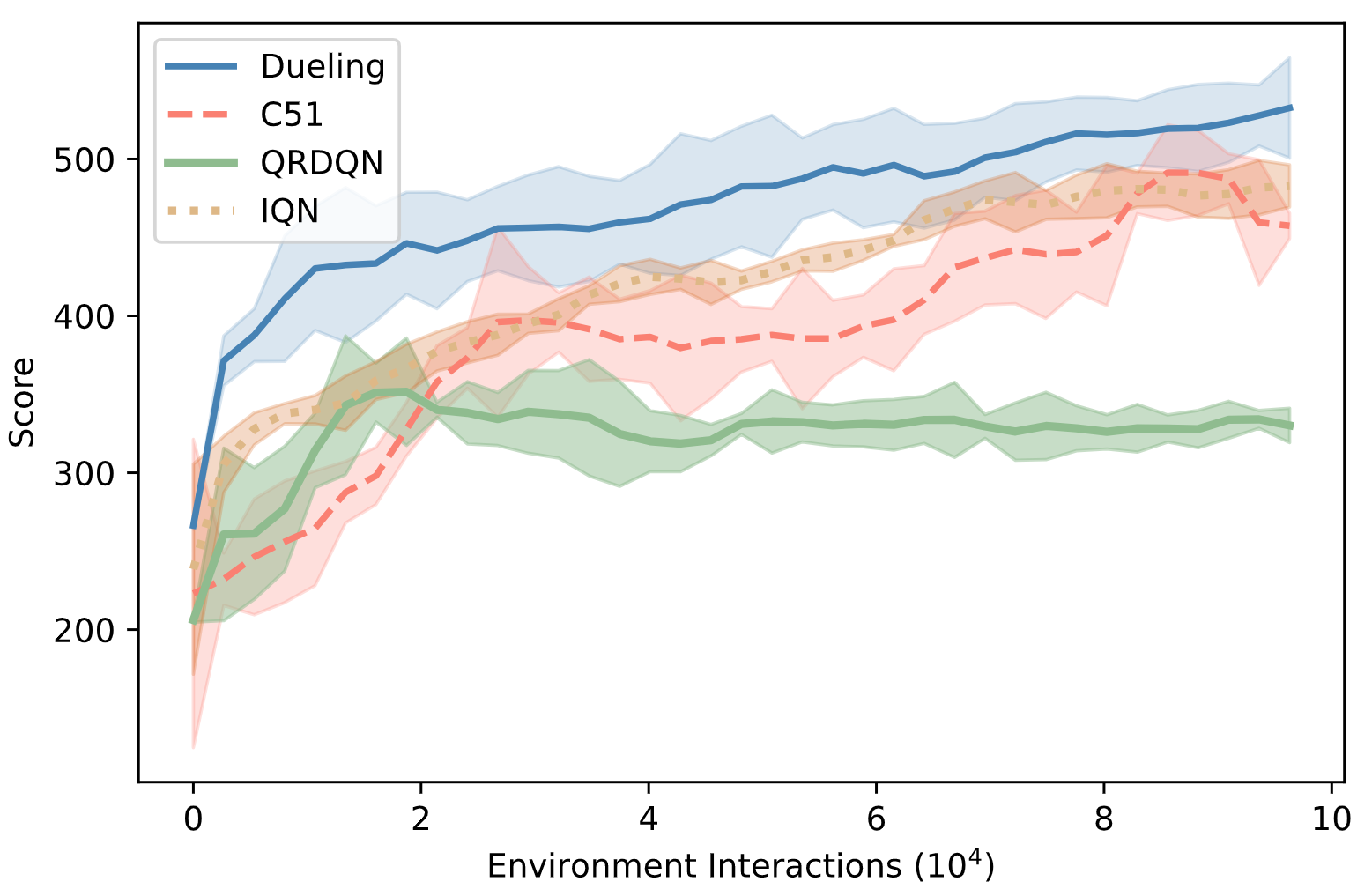}}{\scriptsize{Assault}}
    \stackunder[1pt]{\includegraphics[scale=0.12]{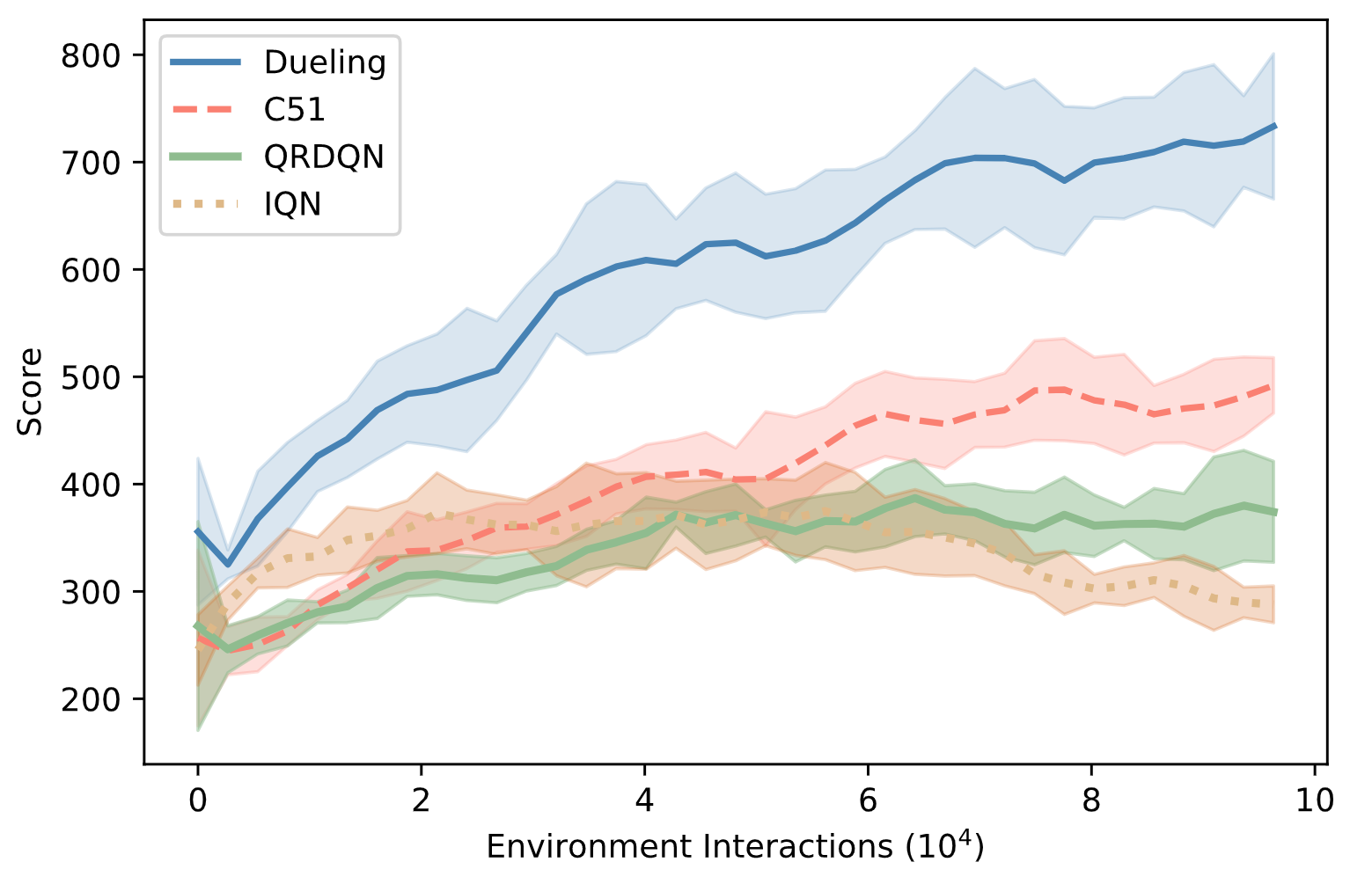}}{\scriptsize{Asterix}}\\
    \vskip-0.02in
    \stackunder[1pt]{\includegraphics[scale=0.12]{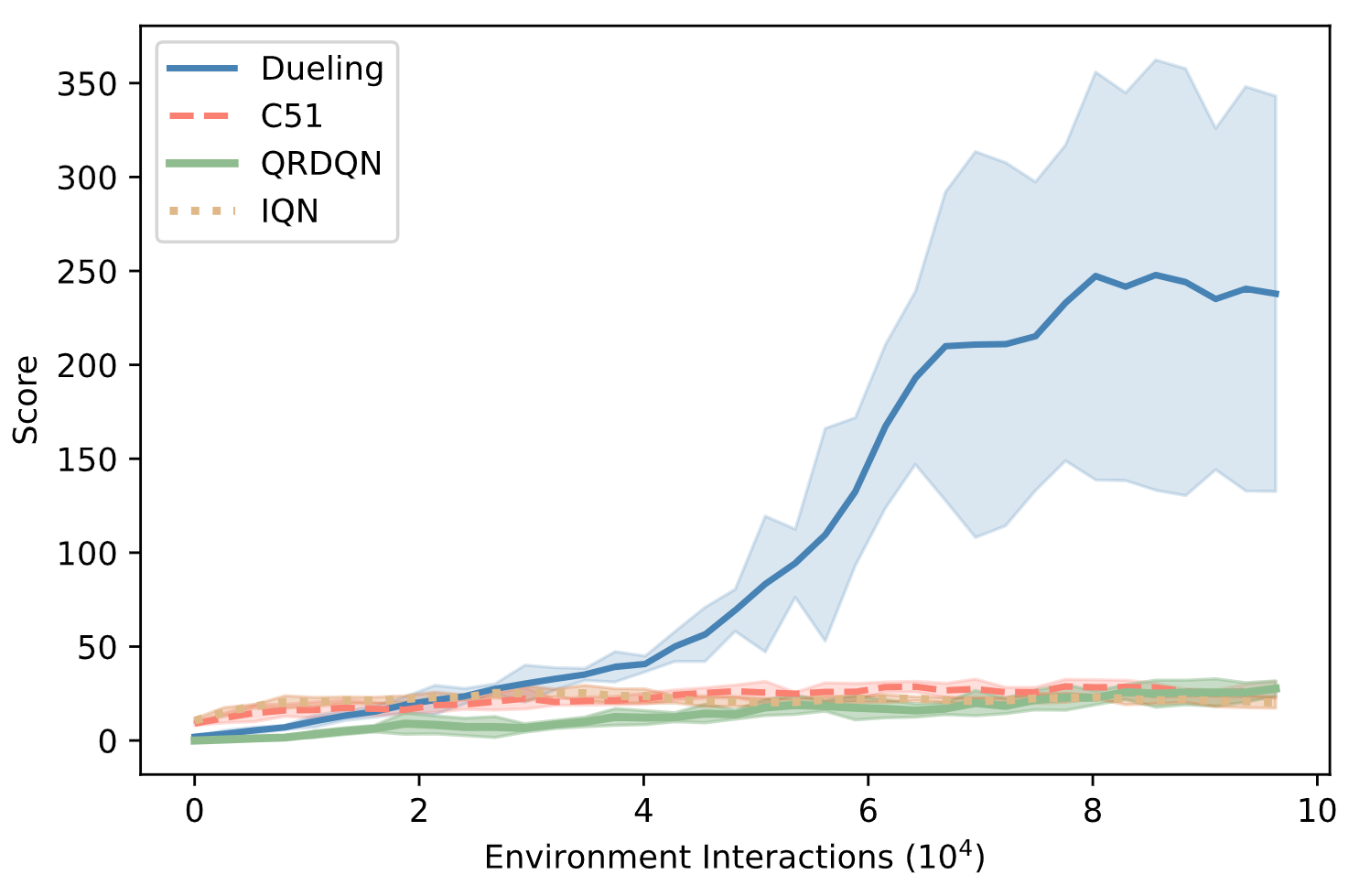}}{\scriptsize{BankHeist}}
    \stackunder[1pt]{\includegraphics[scale=0.12]{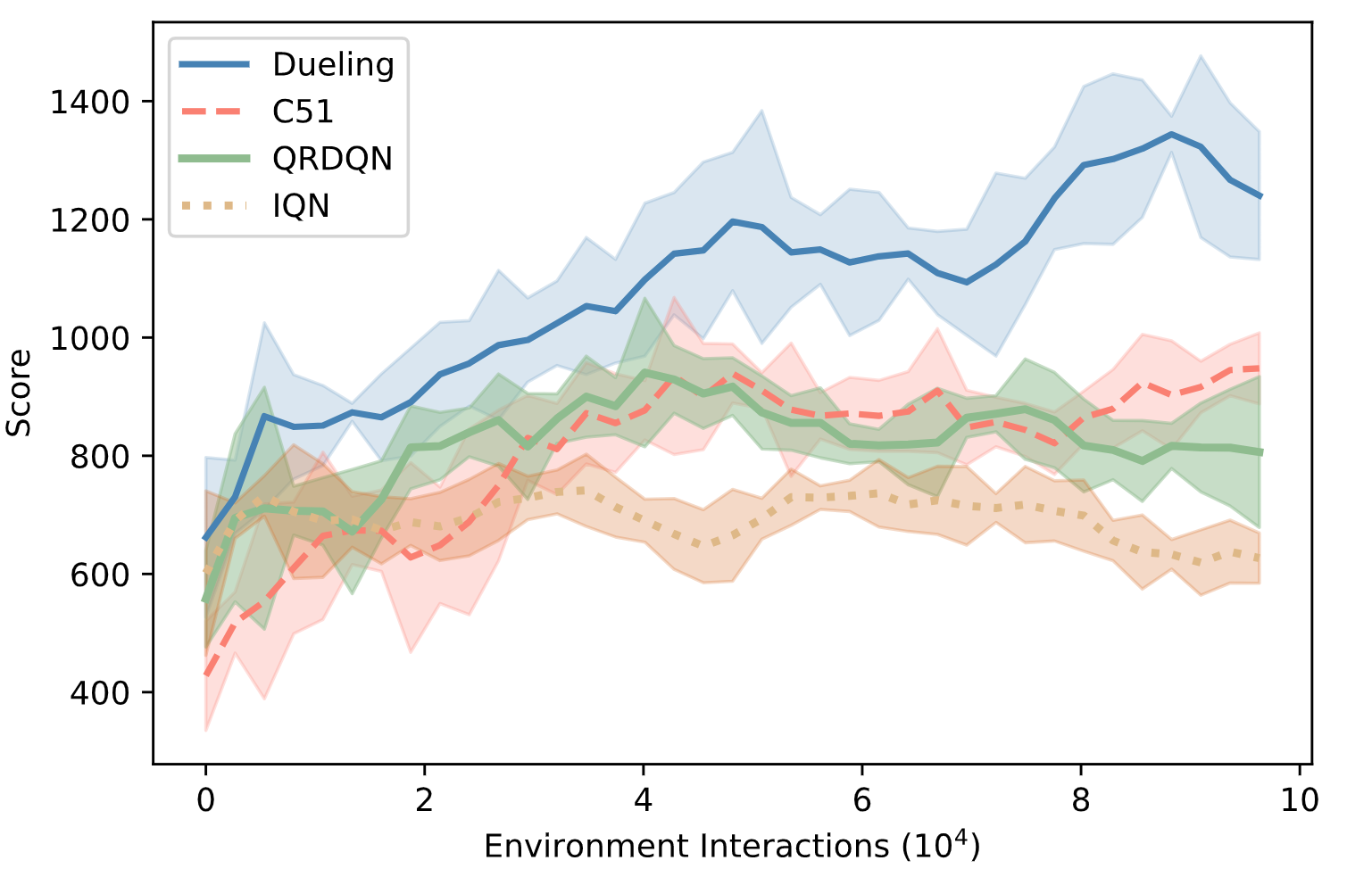}}{\scriptsize{ChopperCommand}}
    \stackunder[1pt]{\includegraphics[scale=0.12]{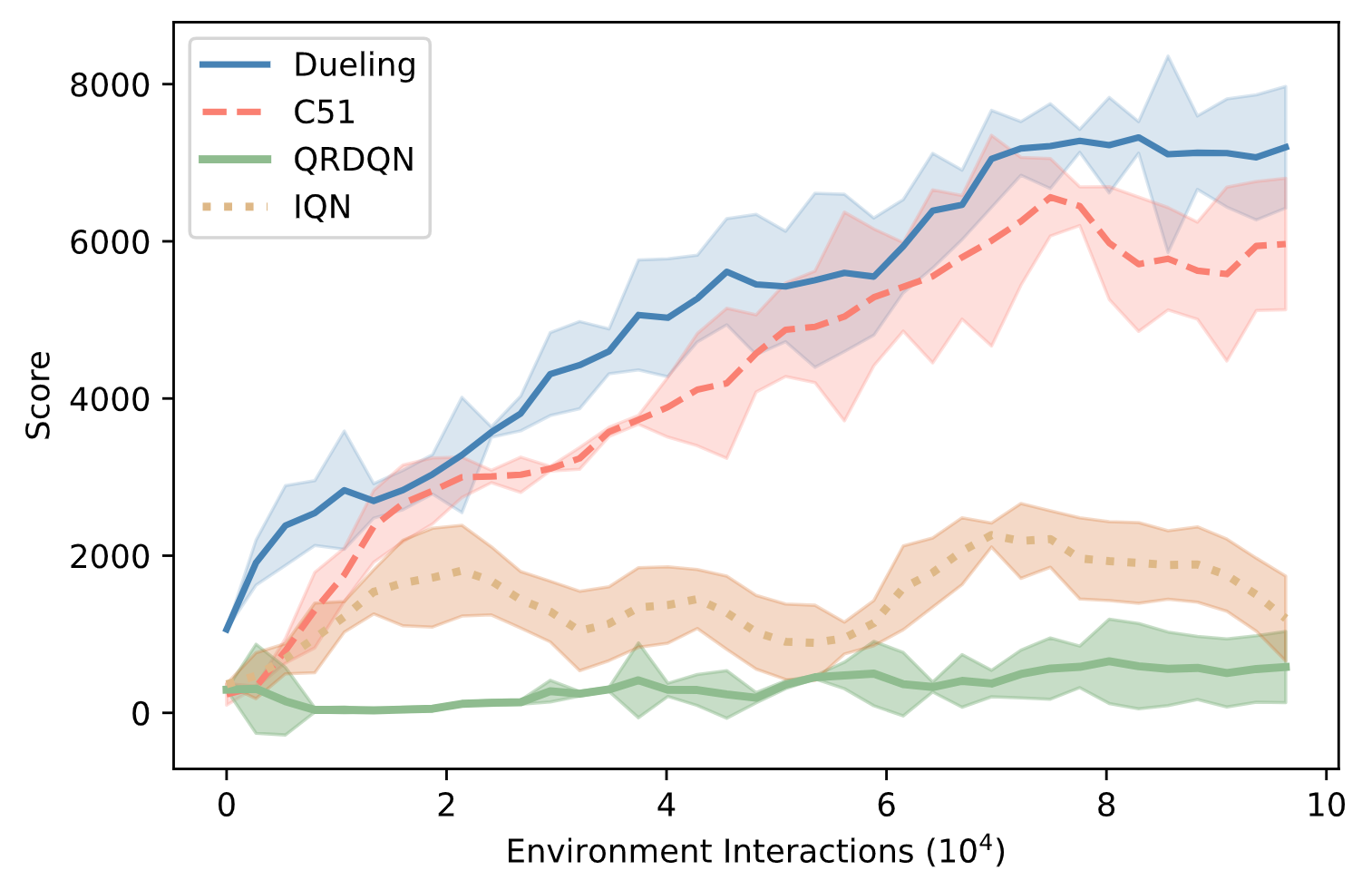}}{\scriptsize{Hero}}
    \stackunder[1pt]{\includegraphics[scale=0.12]{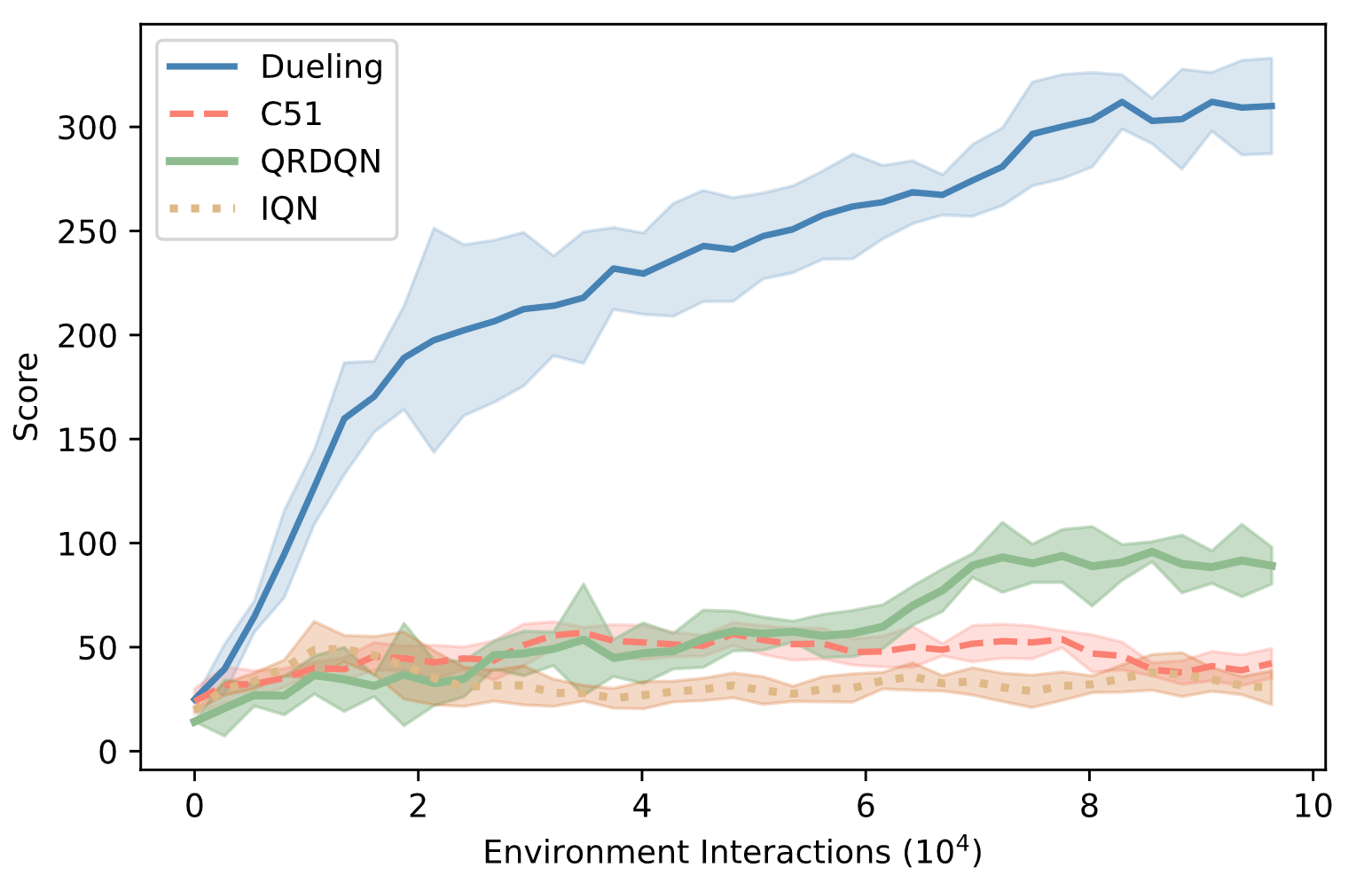}}{\scriptsize{JamesBond}}\\
    \vskip-0.02in
    \stackunder[1pt]{\includegraphics[scale=0.12]{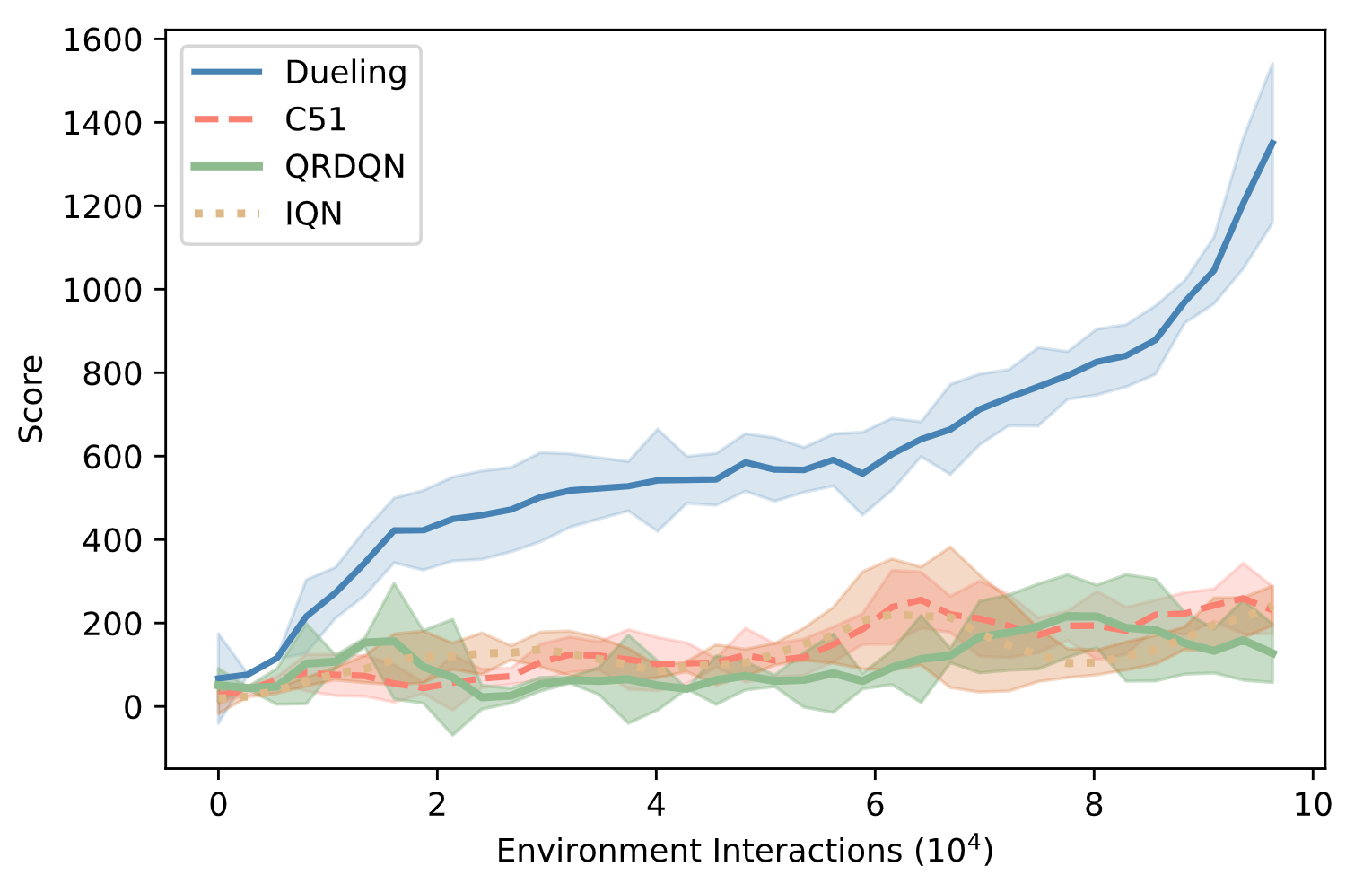}}{\scriptsize{Kangaroo}}
    \stackunder[1pt]{\includegraphics[scale=0.12]{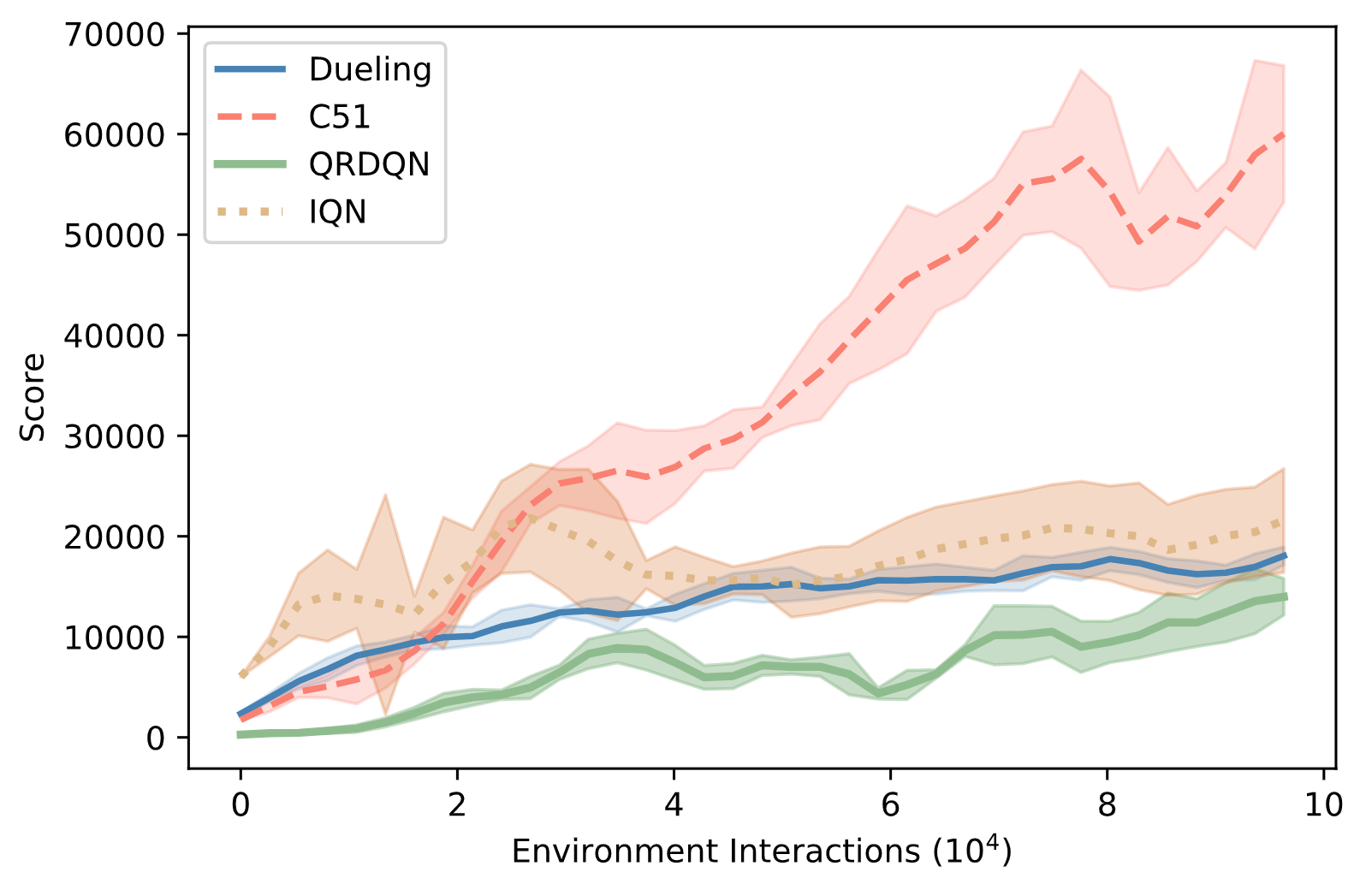}}{\scriptsize{CrazyClimber}}
    \stackunder[1pt]{\includegraphics[scale=0.124]{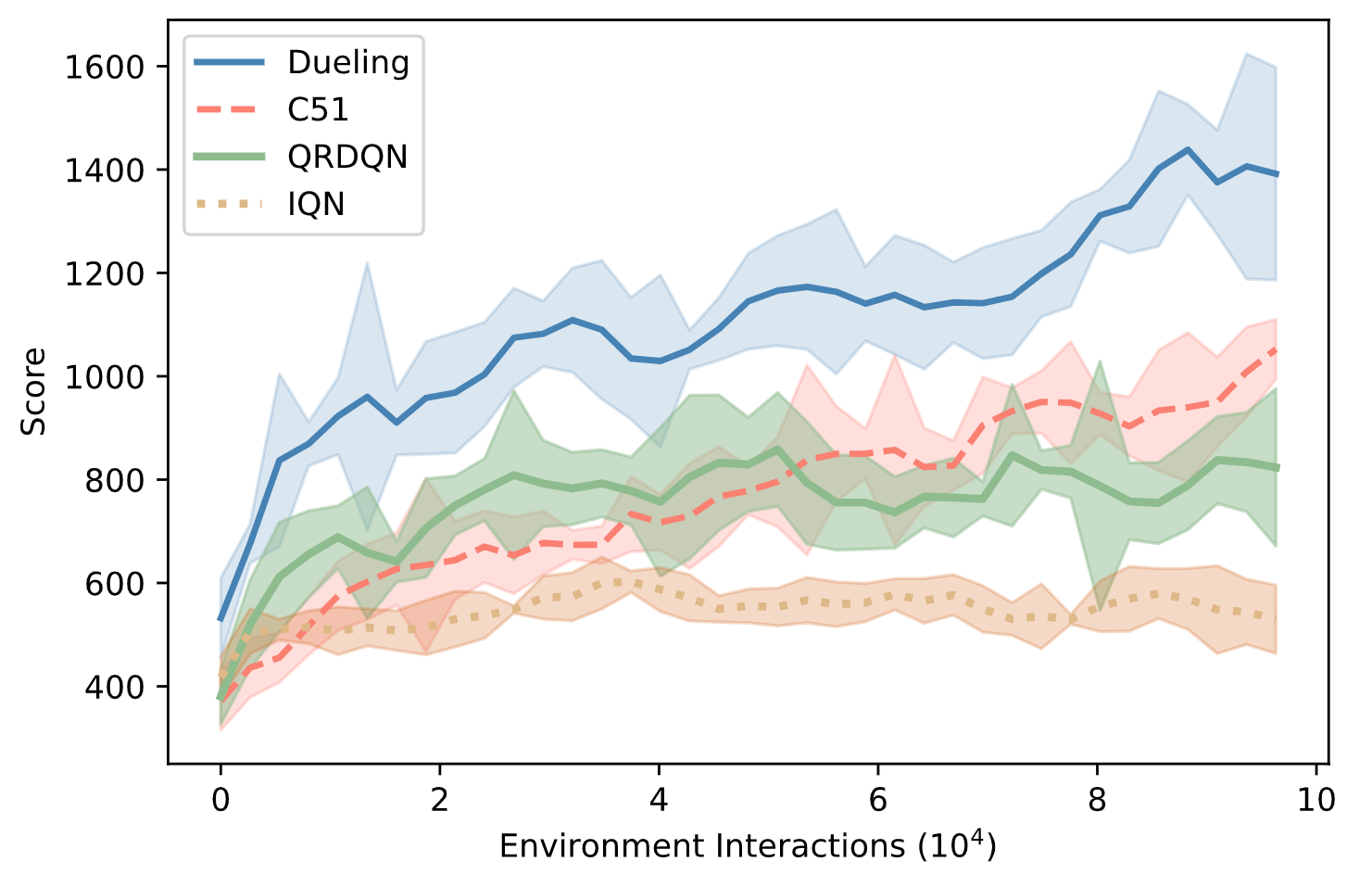}}{\scriptsize{MsPacman}}
    \stackunder[1pt]{\includegraphics[scale=0.12]{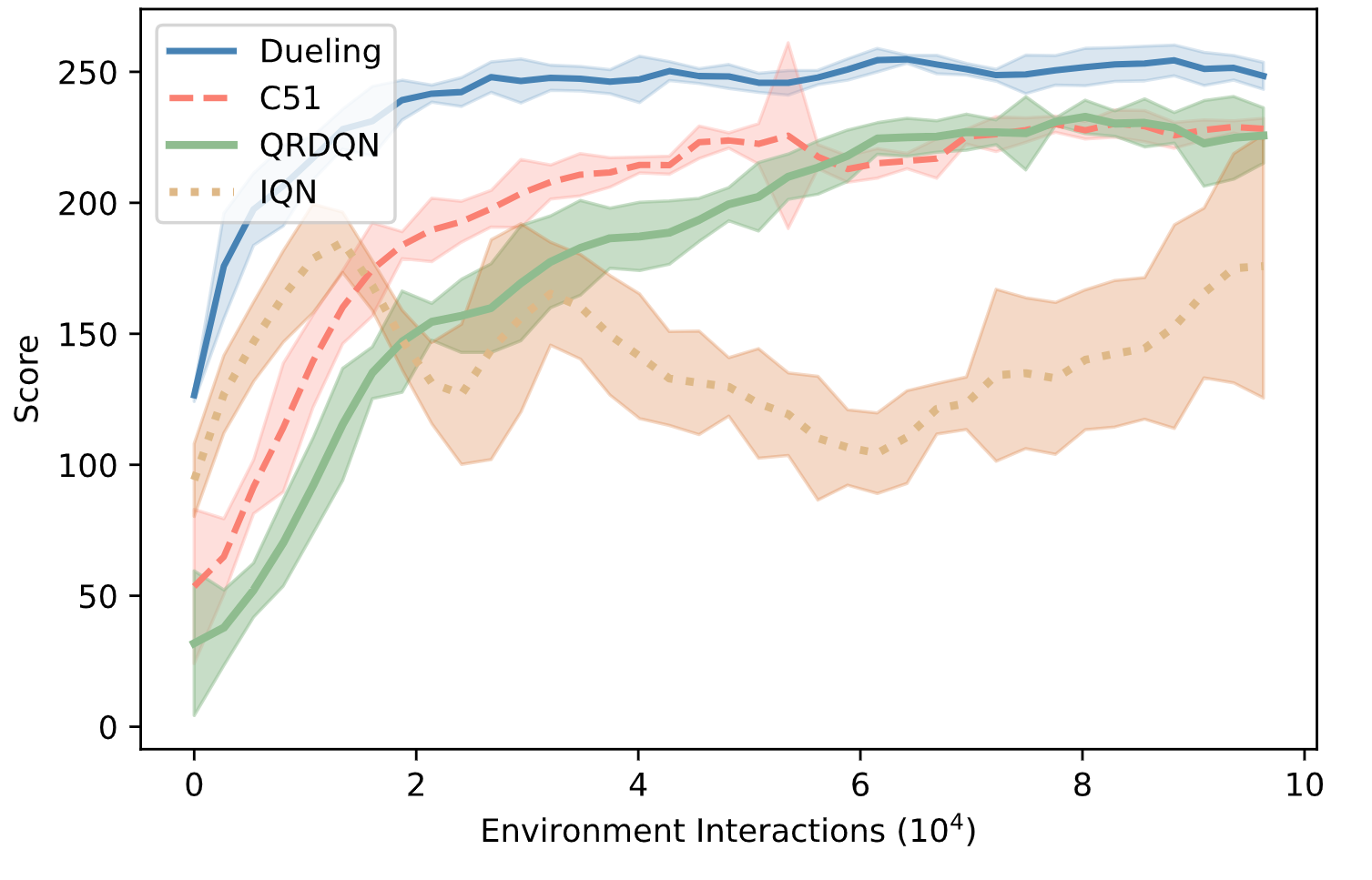}}{\scriptsize{FrostBite}}\\
    \vskip-0.02in
    \stackunder[1pt]{\includegraphics[scale=0.12]{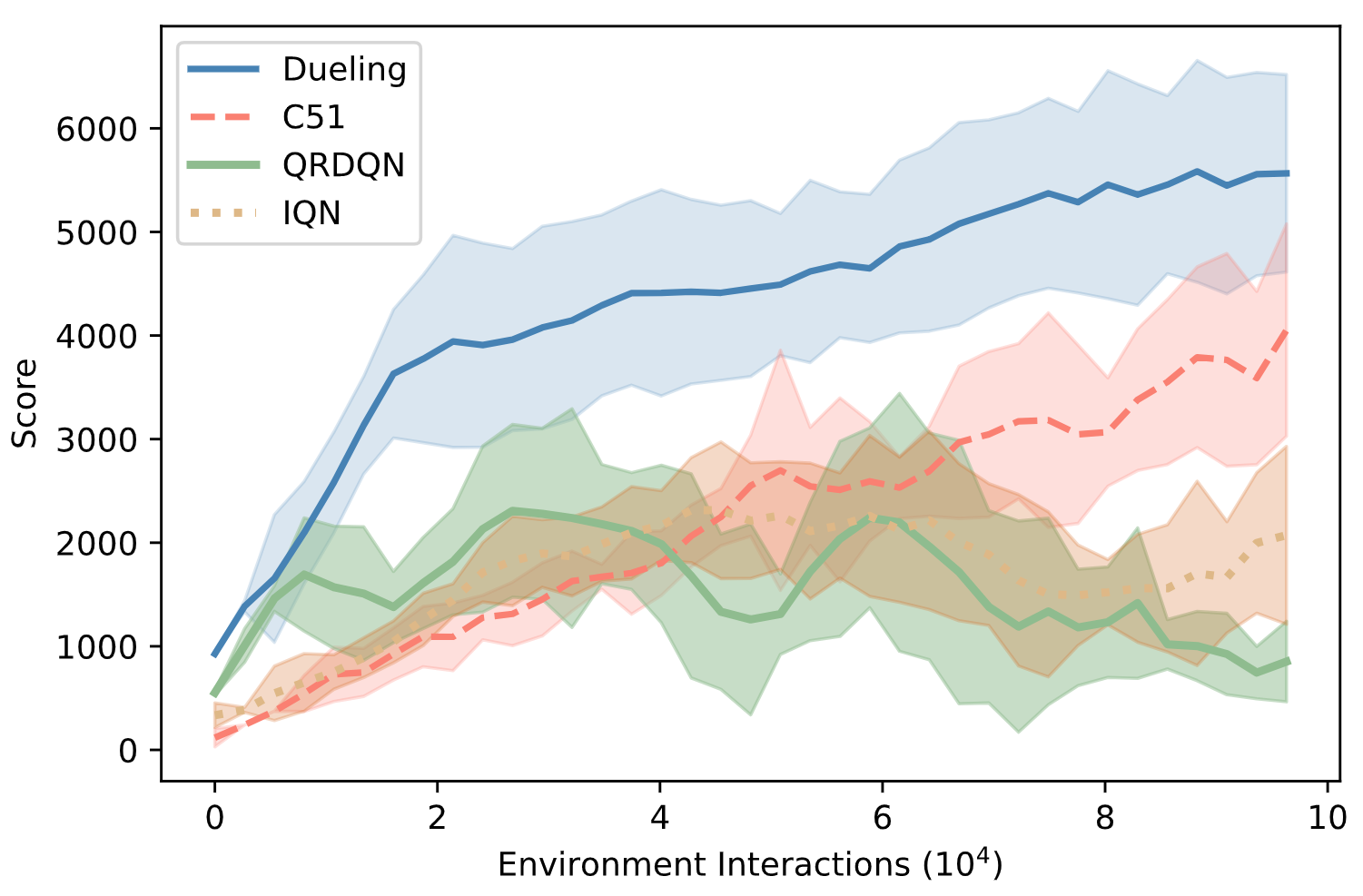}}{\scriptsize{RoadRunner}}
    \stackunder[1pt]{\includegraphics[scale=0.12]{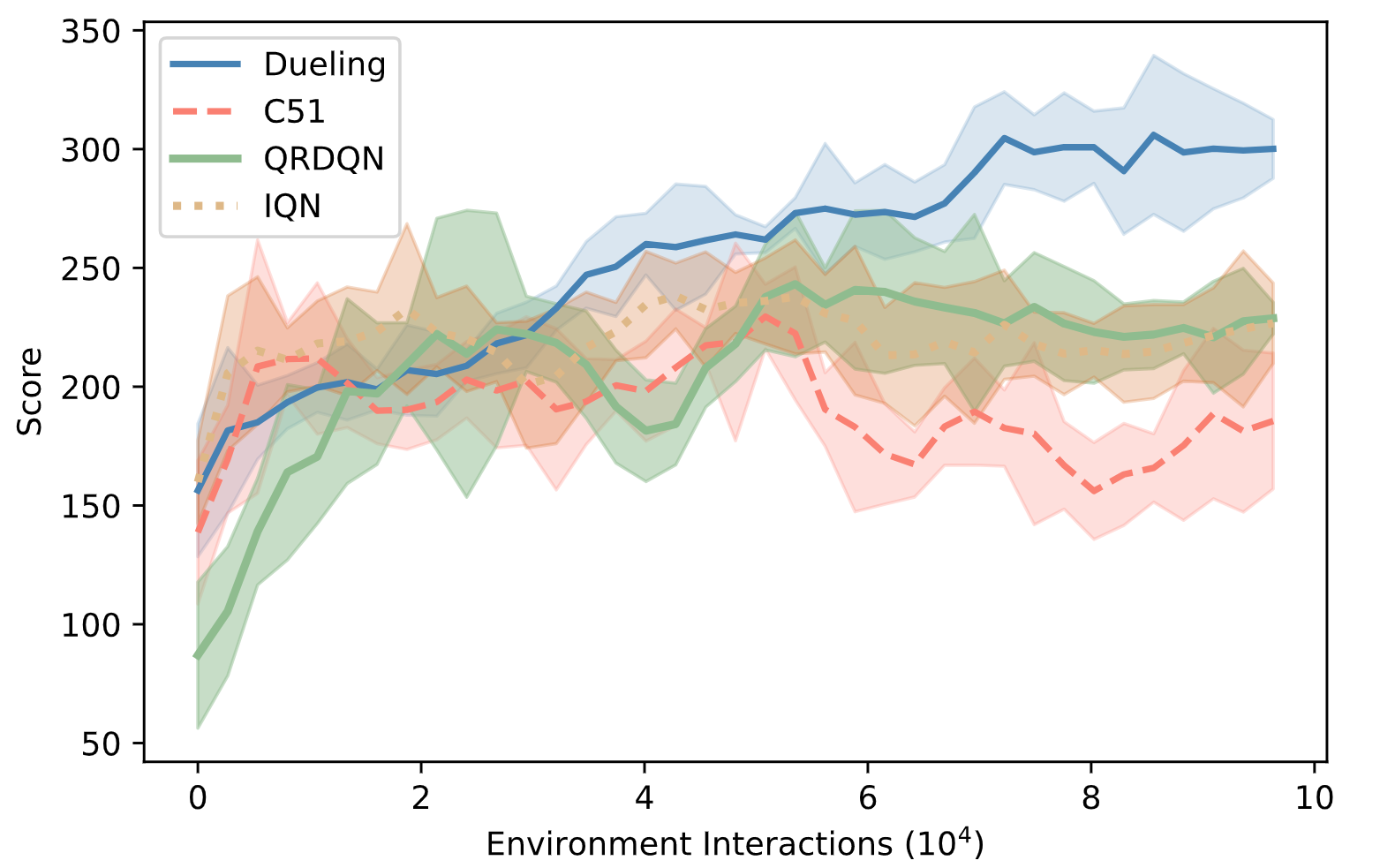}}{\scriptsize{Seaquest}}
    \stackunder[1pt]{\includegraphics[scale=0.12]{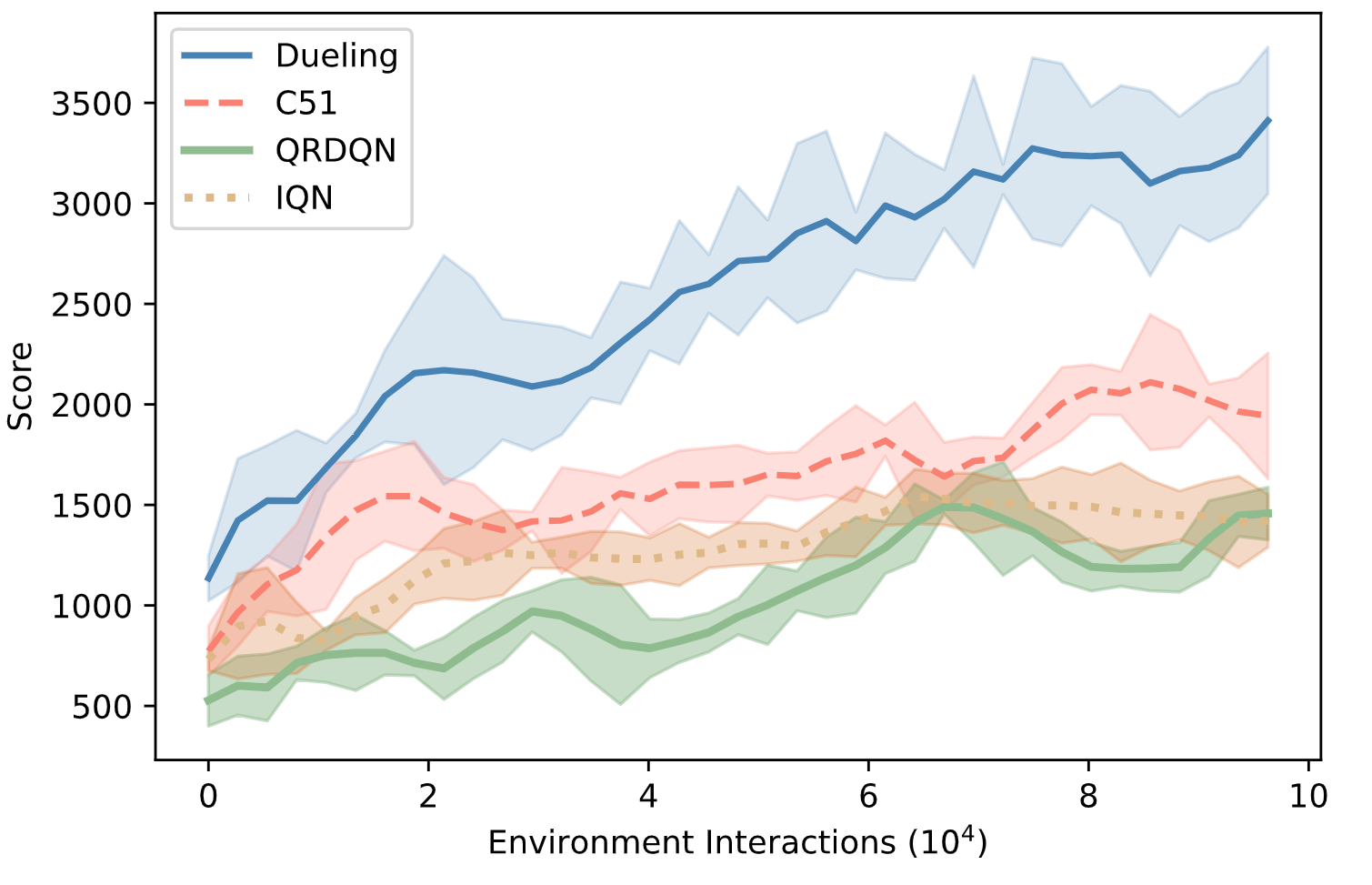}}{\scriptsize{UpNDown}}
    \stackunder[1pt]{\includegraphics[scale=0.12]{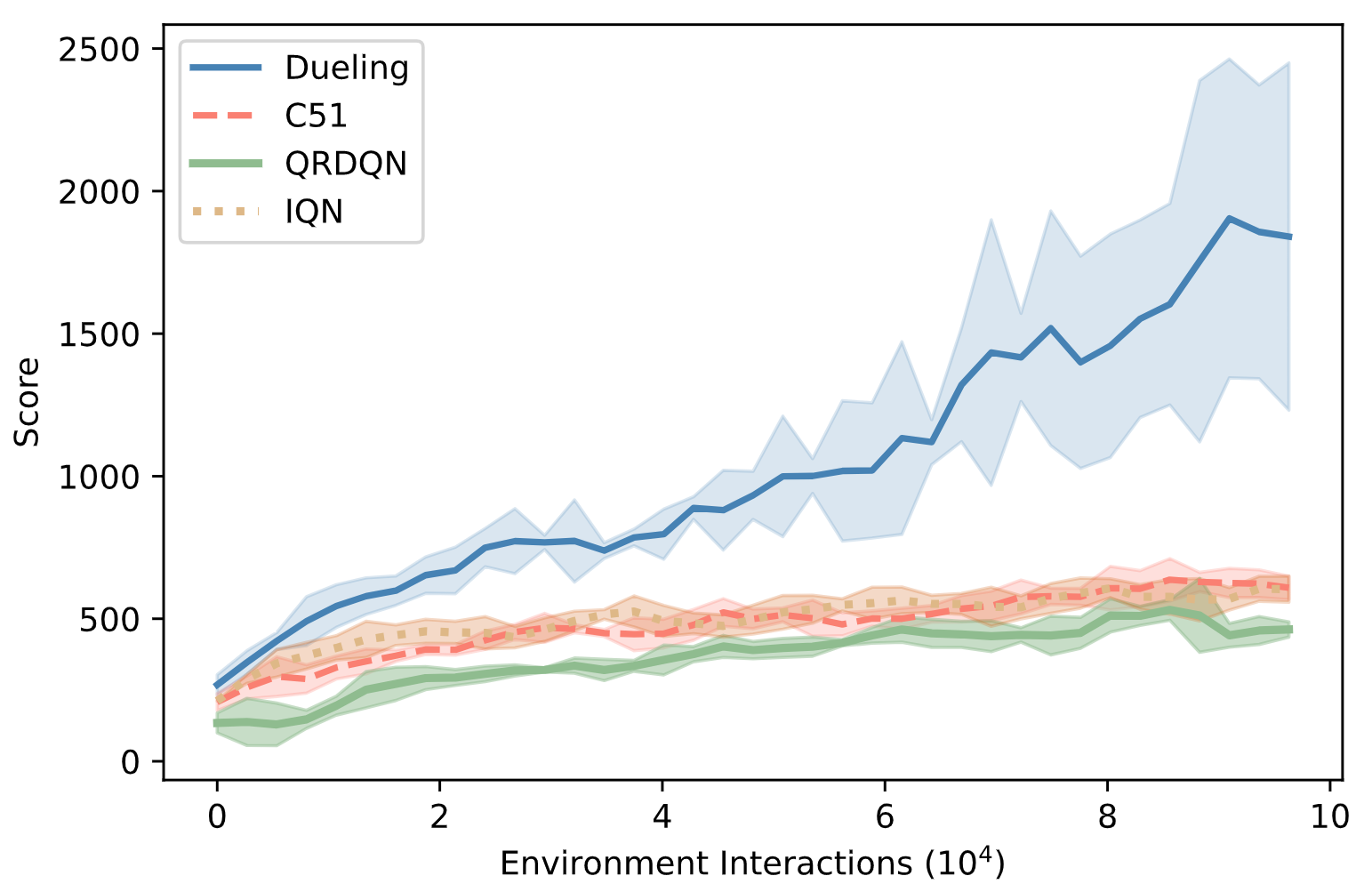}}{\scriptsize{Qbert}}\\
    \end{center}
    \vskip -0.18in
    \caption{The learning curves of Alien, Amidar, Asterix, BankHeist, ChopperCommand, Hero, CrazyClimber, JamesBond, Kangaroo, MsPacman, FrostBite, Qbert, RoadRunner, Seaquest and UpNDown with dueling architecture, C51, I$\mathcal{Q}$N and $\mathcal{Q}$RD$\mathcal{Q}$N algorithms in the Arcade Learning Environment with 100K environment interaction training.}
    \label{100K}
    \end{figure*}

    \vskip0.02in
\noindent \textbf{Implicit Assumptions on Monotonicity Cause Suboptimal and Incorrect Conclusions.}
Our extensive large-scale empirical analysis demonstrates that a major line of research conducted in the past five years resulted in incorrect conclusions. 
We show that a simple baseline algorithm \textbf{from 2016} \cite{wang16}, by a systematic methodological choice was never included in the comparison benchmark, following the implicit assumption that appears in all of the recent line of research that we have discussed in detail in Section \ref{samplecomplexity}.
 We demonstrate that this baseline algorithm in fact performs much better than many recent algorithms that claimed to be better than the baselines, even including algorithms that are specifically built on top of the baseline algorithm.
Figure \ref{100K} reports learning curves for the I$\mathcal{Q}$N, $\mathcal{Q}$RD$\mathcal{Q}$N, dueling and C51 
in the Arcade Learning Environment low-data regime benchmark. 
These results demonstrate that the simple base algorithm dueling performs significantly better
than a series of algorithms that were included in the comparison benchmark which inherently produced higher capacity models 
when the training samples are limited.
Note that DR$\mathcal{Q}$ uses the dueling architecture 
without any high capacity inducing components.
One intriguing takeaway from the results provided in Table \ref{comparison} and Figure \ref{samples}\footnote{DER$^{\textrm{2021}}$ refers to the re-implementation with random seed variations of the original paper data-efficient Rainbow (i.e. DER$^{\textrm{2019}}$) by \citet{hado19}. OTR refers to further implementation of the Rainbow algorithm by \citet{kielak19}.
DR$\mathcal{Q}^{\textrm{NeurIPS}}$ refers to the re-implementation of the original DR$\mathcal{Q}$ algorithm 
with the goal of achieving reproducibility with variation on the number of random seeds \citep{agarwal21}.}
is the fact that the simple baseline dueling algorithm performs 15\% better than the DR$\mathcal{Q}^{\textrm{NeurIPS}}$ implementation, and 11\% less than the DR$\mathcal{Q}^{\textrm{ICLR}}$ implementation instead of 82\% gain reported in the original paper.

\noindent \textbf{Providing Direct Comparison to Core Algorithms.} 
Algorithms that are built on top of a core reinforcement learning algorithm must provide a direct comparison to the algorithm they are built on top of. The case of DR$\mathcal{Q}$ demonstrates the significance of the direct comparison to the core algorithm. As our paper discovers and describes extensively, the monotonicity assumption on the performance ranking across regimes led a line of work to benchmark against certain algorithms in the low-data regime, assuming that if an algorithm has the highest performance in the high-data regime it must have the top-ranked performance in the low-data regime. However, as we pointed out in our theoretical analysis this is a dangerous and incorrect assumption. The results reported in Figure \ref{data} and Figure \ref{samples} demonstrate that these implicit assumptions in fact lead to incorrect and suboptimal conclusions.

\noindent \textbf{Non-Monotonicity of Performance Ranking Across Regimes.}
Table \ref{comparison} reports the 
human normalized median, mean and 20$^\textrm{th}$ percentile 
results over all of the MDPs from the 100K ALE benchmark for D$\mathcal{Q}$N, Double-$\mathcal{Q}$, dueling, C51, $\mathcal{Q}$RD$\mathcal{Q}$N, I$\mathcal{Q}$N and Prior. One important takeaway from the results reported in the Table \ref{comparison} is the fact that one particular algorithm performance profile in 200 million frame training will not directly transfer to the low-data region as predicted by our theoretical analysis in Section \ref{theorysec}.
Figure \ref{all} reports the learning curves of 
human normalized median, mean and 20$^\textrm{th}$ percentile 
for the dueling algorithm, C51, $\mathcal{Q}$RD$\mathcal{Q}$N, and I$\mathcal{Q}$N.
These results once more demonstrate that the performance profile of the simple base algorithm dueling is significantly better than 
any core algorithm 
which inherently produced higher capacity models
that was included in the comparison benchmark of the extensive low-data regime literature 
when the number of environment interactions are limited.

\begin{figure*}[t]
\footnotesize
\vskip -0.1in
\begin{center}
\stackunder[1pt]{\includegraphics[scale=0.193]{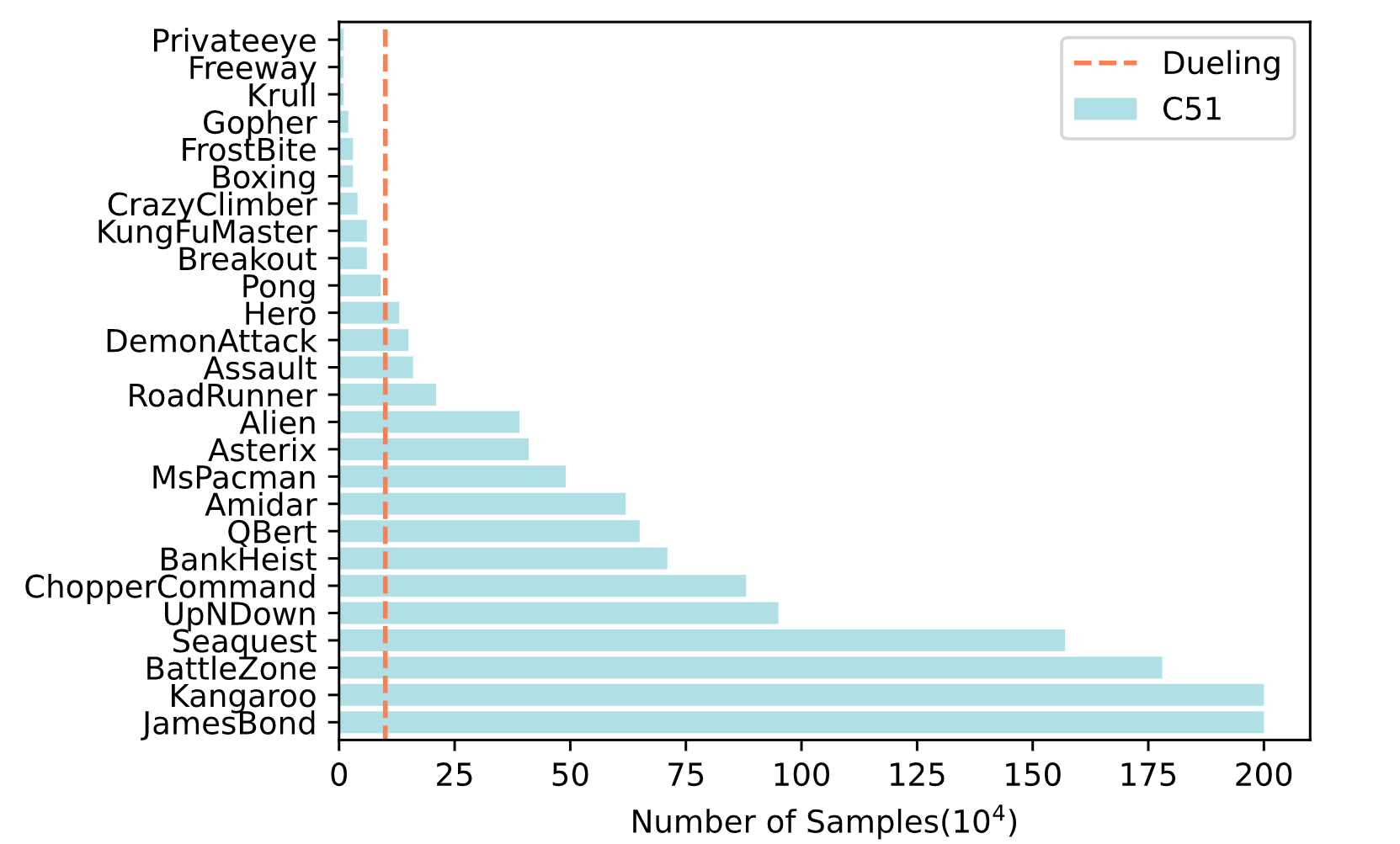}}{Sample Complexity of C51}
\hskip-0.14in
\stackunder[1pt]{\includegraphics[scale=0.199]{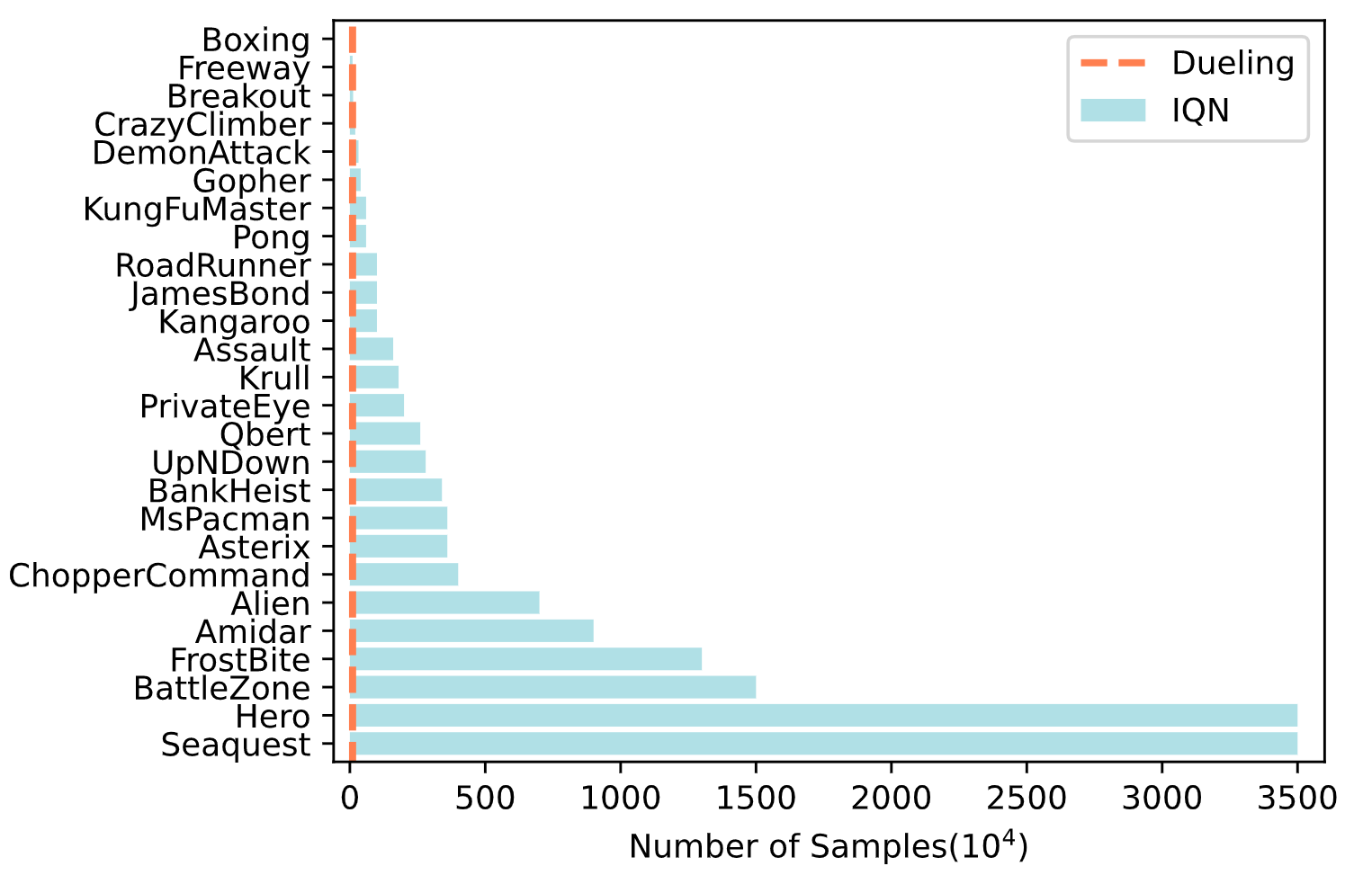}}{Sample Complexity of I$\mathcal{Q}$N}
\stackunder[1pt]{\includegraphics[scale=0.182]{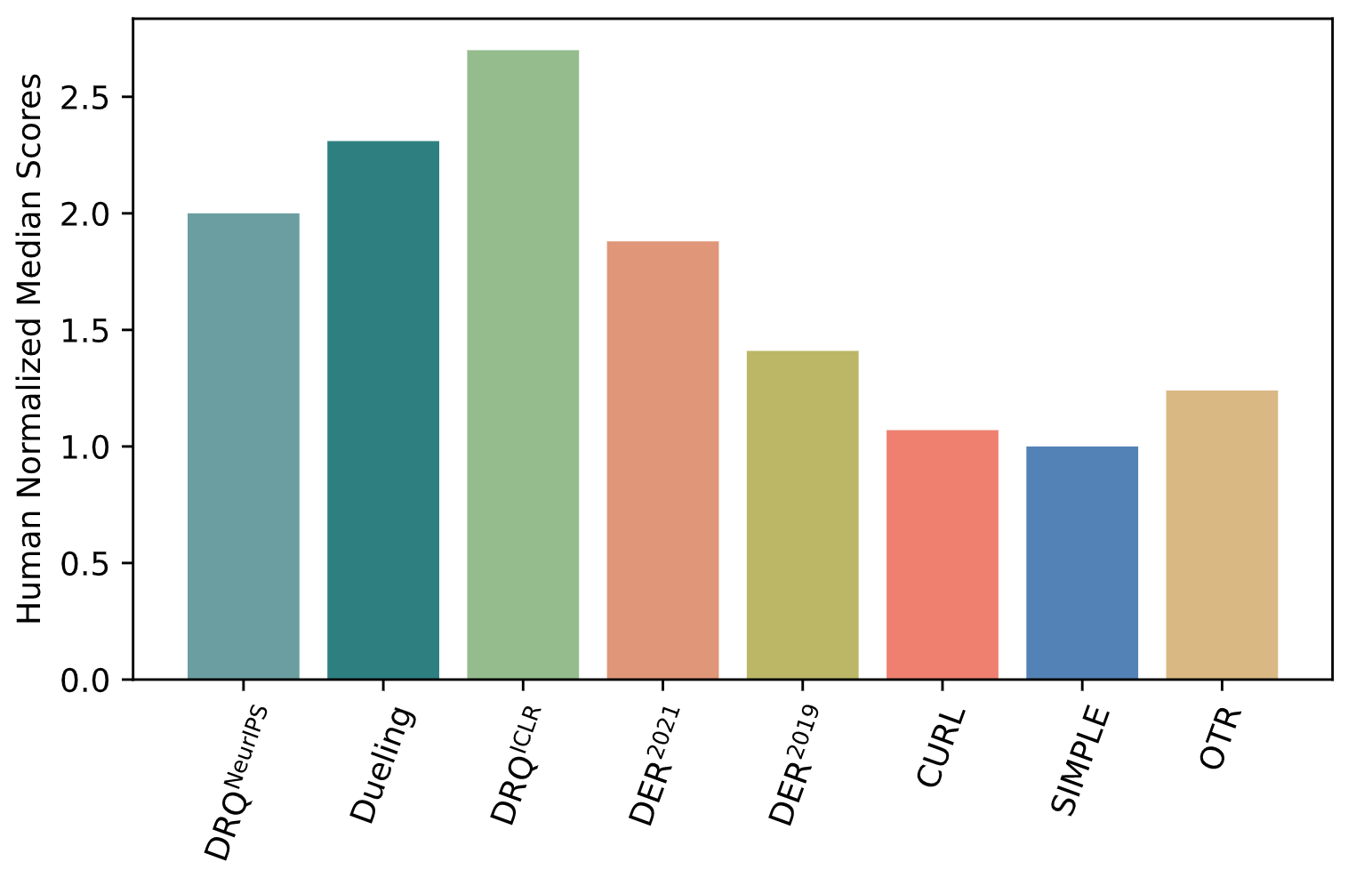}}{Overall Comparison}
\hskip 0.1pt
\end{center}
\vskip -0.18in
\caption{Left: Number of samples, i.e. environment interactions, required by 
the base algorithms that inherently produce higher capacity models
to achieve the performance level achieved by the dueling algorithm. Center: Number of samples required by I$\mathcal{Q}$N to achieve the performance level achieved by dueling. Right: Overall comparison of algorithms recently developed in the low-data regime ALE 100K benchmark to the dueling algorithm that was designed in the high-data region.}
\label{samples}
\end{figure*}
\noindent \textbf{Biases in the Evaluation Criteria.}
The original paper of the DR$\mathcal{Q}^{\textrm{ICLR}}$ algorithm \citep{yarats21} benchmarks against data-efficient Rainbow (DER) \citep{hado19} which inherently learns a higher capacity model.
Our results show that the fact that the original paper that proposed data augmentation for reinforcement learning, i.e. DR$\mathcal{Q}^{\textrm{ICLR}}$, on top of the dueling algorithm did not provide comparisons 
against the core algorithm that they are built on, i.e. dueling \citep{wang16},
resulted in inflated performance profiles for the DR$\mathcal{Q}^{\textrm{ICLR}}$ algorithm.
For a fair, direct and transparent comparison we kept the hyperparameters for the baseline algorithms in the low-data regime exactly the same with the DR$\mathcal{Q}^{\textrm{ICLR}}$ paper (see supplementary material for the full list and high-data regime hyperparameter settings).
More intriguingly, the comparisons provided in the DR$\mathcal{Q}^{\textrm{ICLR}}$ paper to the DER and OTR algorithms report the performance gained by DR$\mathcal{Q}^{\textrm{ICLR}}$ over DER is 82\% and over OTR is 35\%.
However, if a direct comparison is made to the simple dueling algorithm as Table \ref{comparison} demonstrates 
the performance gain is utterly restricted to \textbf{11\%}.
Moreover, when it is compared to the reproduced results of DR$\mathcal{Q}^{\textrm{NeurIPS}}$ our results reveal that in fact 
there is a performance decrease due to utilizing DR$\mathcal{Q}$ over dueling.
Thus, while our paper introduces the foundations on the non-monotonicity of the performance profiles from large-data regime to low-data regime, it further provides the basis on how we can compare algorithms with 
a principled approach and scientific rigor 
allowing more concrete and accurate evaluation across data-regimes.

\noindent \textbf{Theoretical Analysis and the Inherent Bellman Error vs Dimensionality.}
The right and center plots of Figure \ref{data} report regret curves corresponding to the theoretical analysis in Theorem \ref{prop:lowdatavsasymp} for various choices of the feature dimensionality $d$ and the inherent Bellman error $\mathcal{I}$. In particular, the center plot shows the low-data regime where the number of episodes $\mathcal{K} < 1000$, while the right plot shows the high-data regime where $\mathcal{K}$ is as large as $500000$. Notably, the relative ordering of the regret across the different choices of $d$ and $\mathcal{I}$ is completely reversed in the high-data regime when compared to the low-data regime.
Recall from Theorem \ref{thm:optregretalg} that the inherent Bellman error is a measure of the accuracy of function approximation under the Bellman operator corresponding to an MDP. Thus, the varying values of $\mathcal{I}$ and $d$ in Figure \ref{data} correspond to a natural setting where increasing the number of model parameters (i.e. increasing $d$) corresponds to an increase in the accuracy of function approximation (i.e. a decrease in $\mathcal{I}$).
Thus the results reported in Figure \ref{data} demonstrate that, even in the natural setting where increased model capacity leads to increased accuracy, there can be a complete reversal in the ordering of algorithm performance between the low and high-data regimes.
Figure \ref{samples} reports results on the number of samples required for training with the baseline algorithm 
that inherently produces higher capacity models
to reach the same performance levels achieved by the dueling algorithm for every MDP from ALE low-data regime benchmark. 
These results once more demonstrate that to reach the same performance levels with the dueling algorithm, baseline algorithms that
inherently learn higher capacity models
require orders of magnitude more samples to train on. 
As discussed in Section \ref{samplecomplexity}, more complex representations for broader classes of distributions come at the cost of a higher sample complexity required for learning.
One intriguing fact is that the original SimPLE paper in the low-data regime benchmarked against the Rainbow algorithm  
which is essentially a higher capacity model designed in the high-data regime 
by having the implicit assumption that the state-of-the art performance profile must transfer monotonically to the low-data regime. These instances of implicit assumptions also occur in DR$\mathcal{Q}^{\textrm{ICLR}}$, CURL, SPR and Efficient-Zero \cite{ye21} even when comparisons are made for more advanced algorithms such as MuZero.

\noindent \textbf{Datasets are Created and Founded on Implicit Assumptions.}
Thus far we have discussed the pivotal role of implicit assumptions on the algorithmic comparisons and developing baselines in deep reinforcement learning. However, this issue further extends back to even how the entire low-data regime benchmark was established, i.e. ALE 100K. The ALE 100K was initially created to allow researchers to work on a subset of games instead of full set of games used in the high-data regime \citep{kaiser20}, and this benchmark is currently used by any algorithm developed for the low data regime. However, the entire ALE 100K benchmark was in fact built on the selection bias of choosing games that performed better either with the proposed algorithm of the paper that proposed the entire benchmark \citep{kaiser20}, or with Rainbow, which we extensively demonstrated throughout the paper is an algorithm that is subjected to the implicit assumption bias on monotonicity across regimes. Thus the issues we explicitly discover and analyze in our paper are not limited to baselines but further extend to canonical benchmarks that we evaluate reinforcement learning algorithms on.
Our paper discovers that the canonical methodological choices made in a major line of deep reinforcement learning research that is based on these implicit assumptions, give incorrect signals on why and what makes these algorithms work, 
and hence affect future research directions while misdirecting the possible current research efforts from ideas that could have worked during the algorithm design process.

\section{Conclusion}

In this paper we aimed to answer the following questions: 
\emph{(i) How are the scaling laws of reinforcement learning formally characterized with respect to capacity and complexity?}
\emph{(ii) What are the canonical methodological choices that fundamentally affect the progress in deep reinforcement learning research?} and
\emph{(iii) What is the underlying theoretical relationship between monotonicity, the performance profiles and sample complexity regimes?}
To be able to answer these questions we provide theoretical analysis on the sample complexity of the baseline deep reinforcement learning algorithms.
We conduct extensive experiments both in the low-data regime 100K Arcade Learning Environment and high-data regime baseline 200 million frame training.
Both theoretical and empirical analysis provided in our paper demonstrate that the performance profiles of deep reinforcement learning algorithms do not have a monotonic relationship across sample complexity regimes.
Our analysis reveals that the underlying assumption of the monotonic relationship of the performance characteristics and the sample complexity regimes is currently present in a major line of research including many recent state-of-the-art studies and this implicit assumption led these studies to result in incorrect conclusions. 
Our paper demonstrates that several baseline $\mathcal{Q}$ algorithms perform better than a line of recent algorithms claimed to be the state-of-the-art.
Our paper establishes a principled analysis of deep reinforcement learning that characterizes the fundamental relationship between
scaling, capacity and complexity.

\bibliography{example_paper}

@article{julian20,
  author       = {Julian Schrittwieser and
                  Ioannis Antonoglou and
                  Thomas Hubert and
                  Karen Simonyan and
                  Laurent Sifre and
                  Simon Schmitt and
                  Arthur Guez and
                  Edward Lockhart and
                  Demis Hassabis and
                  Thore Graepel and
                  Timothy P. Lillicrap and
                  David Silver},
  title        = {Mastering Atari, Go, chess and shogi by planning with a learned model},
  journal      = {Nat.},
  volume       = {588},
  number       = {7839},
  pages        = {604--609},
  year         = {2020},
  url          = {https://doi.org/10.1038/s41586-020-03051-4},
  doi          = {10.1038/S41586-020-03051-4},
  bibsource    = {dblp computer science bibliography, https://dblp.org}
}

@inproceedings{kaptur23,
  author       = {Steven Kapturowski and
                  Victor Campos and
                  Ray Jiang and
                  Nemanja Rakicevic and
                  Hado van Hasselt and
                  Charles Blundell and
                  Adri{\`{a}} Puigdom{\`{e}}nech Badia},
  title        = {Human-level Atari 200x faster},
  booktitle    = {The Eleventh International Conference on Learning Representations,
                  {ICLR} 2023, Kigali, Rwanda, May 1-5, 2023},
  year         = {2023},
}

@inproceedings{korkmaz24icml,
  author       = {Ezgi Korkmaz},
  title        = {Understanding and {D}iagnosing {D}eep {R}einforcement {L}earning},
  booktitle    = {International Conference on Machine Learning, {ICML} 2024},
  year         = {2024},
}

@inproceedings{kaiser20,
  author       = {Lukasz Kaiser and
                  Mohammad Babaeizadeh and
                  Piotr Milos and
                  Blazej Osinski and
                  Roy H. Campbell and
                  Konrad Czechowski and
                  Dumitru Erhan and
                  Chelsea Finn and
                  Piotr Kozakowski and
                  Sergey Levine and
                  Afroz Mohiuddin and
                  Ryan Sepassi and
                  George Tucker and
                  Henryk Michalewski},
  title        = {Model Based Reinforcement Learning for Atari},
  booktitle    = {8th International Conference on Learning Representations, {ICLR} 2020 {[Spotlight Presentation]}},
  year         = {2020},
}

@inproceedings{dab18,
  author    = {Will Dabney and
               Georg Ostrovski and
               David Silver and
               R{\'{e}}mi Munos},
  editor    = {Jennifer G. Dy and
               Andreas Krause},
  title     = {Implicit Quantile Networks for Distributional Reinforcement Learning},
  booktitle = {Proceedings of the 35th International Conference on Machine Learning,
               {ICML} 2018, Stockholmsm{\"{a}}ssan, Stockholm, Sweden, July
               10-15, 2018},
  series    = {Proceedings of Machine Learning Research},
  volume    = {80},
  pages     = {1104--1113},
  publisher = {{PMLR}},
  year      = {2018},
}

@inproceedings{hessel21,
  author    = {Matteo Hessel and
               Ivo Danihelka and
               Fabio Viola and
               Arthur Guez and
               Simon Schmitt and
               Laurent Sifre and
               Theophane Weber and
               David Silver and
               Hado van Hasselt},
  editor    = {Marina Meila and
               Tong Zhang},
  title     = {Muesli: Combining Improvements in Policy Optimization},
  booktitle = {Proceedings of the 38th International Conference on Machine Learning,
               {ICML} 2021, 18-24 July 2021, Virtual Event},
  series    = {Proceedings of Machine Learning Research},
  volume    = {139},
  pages     = {4214--4226},
  publisher = {{PMLR}},
  year      = {2021},
}

@article{kielak19,
author = {Kacper Piotr Kielak},
journal = {CoRR},
title = {Do recent advancements in model-based deep reinforcement learning really improve data efficiency?},
year = {2019}
}

@article{korkmaz25,
author = {Ezgi Korkmaz},
journal = {Advances in Neural Information Processing Systems 39: Annual Conference
                  on Neural Information Processing Systems 2025, {NeurIPS} 2025 {[Spotlight Presentation]}},
title = {Counteractive RL: Rethinking Core Principles for Efficient and Scalable Deep Reinforcement Learning},
year = {2025}
}

@inproceedings{agarwal21,
  author       = {Rishabh Agarwal and
                  Max Schwarzer and
                  Pablo Samuel Castro and
                  Aaron C. Courville and
                  Marc G. Bellemare},
  editor       = {Marc'Aurelio Ranzato and
                  Alina Beygelzimer and
                  Yann N. Dauphin and
                  Percy Liang and
                  Jennifer Wortman Vaughan},
  title        = {Deep Reinforcement Learning at the Edge of the Statistical Precipice},
  booktitle    = {Advances in Neural Information Processing Systems 34: Annual Conference
                  on Neural Information Processing Systems 2021, NeurIPS 2021},
  pages        = {29304--29320},
  year         = {2021},
}

@inproceedings{will18,
  author    = {Will Dabney and
               Mark Rowland and
               Marc G. Bellemare and
               R{\'{e}}mi Munos},
  editor    = {Sheila A. McIlraith and
               Kilian Q. Weinberger},
  title     = {Distributional Reinforcement Learning With Quantile Regression},
  booktitle = {Proceedings of the Thirty-Second {AAAI} Conference on Artificial Intelligence,
               (AAAI-18), the 30th innovative Applications of Artificial Intelligence
               (IAAI-18), and the 8th {AAAI} Symposium on Educational Advances in
               Artificial Intelligence (EAAI-18), New Orleans, Louisiana, USA, February
               2-7, 2018},
  pages     = {2892--2901},
  publisher = {{AAAI} Press},
  year      = {2018},
}

@inproceedings{ZanetteLKB20,
  author    = {Andrea Zanette and
               Alessandro Lazaric and
               Mykel J. Kochenderfer and
               Emma Brunskill},
  title     = {Learning Near Optimal Policies with Low Inherent Bellman Error},
  booktitle = {Proceedings of the 37th International Conference on Machine Learning,
               {ICML} 2020, 13-18 July 2020, Virtual Event},
  series    = {Proceedings of Machine Learning Research},
  volume    = {119},
  pages     = {10978--10989},
  publisher = {{PMLR}},
  year      = {2020},
  url       = {http://proceedings.mlr.press/v119/zanette20a.html},
  bibsource = {dblp computer science bibliography, https://dblp.org}
}

@inproceedings{yarats21,
  author    = {Denis Yarats and
               Ilya Kostrikov and
               Rob Fergus},
  title     = {Image Augmentation Is All You Need: Regularizing Deep Reinforcement
               Learning from Pixels},
  booktitle = {9th International Conference on Learning Representations, {ICLR} 2021,
               Virtual Event, Austria, May 3-7, 2021 {[Spotlight Presentation]}},
  year      = {2021},
}

@inproceedings{hado19,
  author    = {Hado van Hasselt and
               Matteo Hessel and
               John Aslanides},
  editor    = {Hanna M. Wallach and
               Hugo Larochelle and
               Alina Beygelzimer and
               Florence d'Alch{\'{e}}{-}Buc and
               Emily B. Fox and
               Roman Garnett},
  title     = {When to use parametric models in reinforcement learning?},
  booktitle = {Advances in Neural Information Processing Systems 32: Annual Conference
               on Neural Information Processing Systems 2019, NeurIPS 2019, December
               8-14, 2019, Vancouver, BC, Canada},
  pages     = {14322--14333},
  year      = {2019},
}

@inproceedings{hado18,
  author    = {Matteo Hessel and
               Joseph Modayil and
               Hado van Hasselt and
               Tom Schaul and
               Georg Ostrovski and
               Will Dabney and
               Dan Horgan and
               Bilal Piot and
               Mohammad Gheshlaghi Azar and
               David Silver},
  editor    = {Sheila A. McIlraith and
               Kilian Q. Weinberger},
  title     = {Rainbow: Combining Improvements in Deep Reinforcement Learning},
  booktitle = {Proceedings of the Thirty-Second {AAAI} Conference on Artificial Intelligence,
               (AAAI-18), the 30th innovative Applications of Artificial Intelligence
               (IAAI-18), and the 8th {AAAI} Symposium on Educational Advances in
               Artificial Intelligence (EAAI-18), New Orleans, Louisiana, USA, February
               2-7, 2018},
  pages     = {3215--3222},
  publisher = {{AAAI} Press},
  year      = {2018},
}

@article{tom16,
author = {Tom Schaul and John Quan and Ioannis Antonogloua and David Silver},
journal = {International Conference on Learning Representations (ICLR)},
title = {Prioritized Experience Replay},
year = {2016}
}

@article{hado16,
author = {Hasselt, Hado van and Arthur Guez and David Silver},
journal = {Association for the Advancement of Artificial Intelligence (AAAI)},
title = {Deep Reinforcement Learning with Double Q-Learning},
year = {2016}
}

@software{jax2,
  title = {The {D}eep{M}ind {E}cosystem},
  author = {Babuschkin, Igor and Baumli, Kate and Bell, Alison and Bhupatiraju, Surya and Bruce, Jake and Buchlovsky, Peter and Budden, David and Cai, Trevor and Clark, Aidan and Danihelka, Ivo and Fantacci, Claudio and Godwin, Jonathan and Jones, Chris and Hennigan, Tom and Hessel, Matteo and Kapturowski, Steven and Keck, Thomas and Kemaev, Iurii and King, Michael and Martens, Lena and Merzic, Hamza and Mikulik, Vladimir and Norman, Tamara and Quan, John and Papamakarios, George and Ring, Roman and Ruiz, Francisco and Sanchez, Alvaro and Schneider, Rosalia and Sezener, Eren and Spencer, Stephen and Srinivasan, Srivatsan and Stokowiec, Wojciech and Viola, Fabio},
  url = {http://github.com/deepmind},
  year = {2020},
}

@software{optax,
  author = {Matteo Hessel and David Budden and Fabio Viola and Mihaela Rosca
            and Eren Sezener and Tom Hennigan},
  title = {Optax: composable gradient transformation and optimisation},
  url = {http://github.com/deepmind/optax},
  version = {0.1.3},
  year = {2020},
}

@inproceedings{nisan20,
  author       = {Nisan Stiennon and
                  Long Ouyang and
                  Jeffrey Wu and
                  Daniel M. Ziegler and
                  Ryan Lowe and
                  Chelsea Voss and
                  Alec Radford and
                  Dario Amodei and
                  Paul F. Christiano},
  editor       = {Hugo Larochelle and
                  Marc'Aurelio Ranzato and
                  Raia Hadsell and
                  Maria{-}Florina Balcan and
                  Hsuan{-}Tien Lin},
  title        = {Learning to summarize with human feedback},
  booktitle    = {Advances in Neural Information Processing Systems 33: Annual Conference
                  on Neural Information Processing Systems 2020, NeurIPS 2020, December
                  6-12, 2020, virtual},
  year         = {2020},
}

@inproceedings{lee24,
  author       = {Harrison Lee and
                  Samrat Phatale and
                  Hassan Mansoor and
                  Thomas Mesnard and
                  Johan Ferret and
                  Kellie Lu and
                  Colton Bishop and
                  Ethan Hall and
                  Victor Carbune and
                  Abhinav Rastogi and
                  Sushant Prakash},
  title        = {{RLAIF} vs. {RLHF:} Scaling Reinforcement Learning from Human Feedback
                  with {AI} Feedback},
  booktitle    = {Forty-first International Conference on Machine Learning, {ICML} 2024,
                  Vienna, Austria, July 21-27, 2024},
  year         = {2024},
}

@inproceedings{hasselt10,
  author    = {Hado van Hasselt},
  editor    = {John D. Lafferty and
               Christopher K. I. Williams and
               John Shawe{-}Taylor and
               Richard S. Zemel and
               Aron Culotta},
  title     = {Double Q-learning},
  booktitle = {Advances in Neural Information Processing Systems 23: 24th Annual
               Conference on Neural Information Processing Systems 2010. Proceedings
               of a meeting held 6-9 December 2010, Vancouver, British Columbia,
               Canada},
  pages     = {2613--2621},
  publisher = {Curran Associates, Inc.},
  year      = {2010},
}

@inproceedings{ye21,
  author       = {Weirui Ye and
                  Shaohuai Liu and
                  Thanard Kurutach and
                  Pieter Abbeel and
                  Yang Gao},
  title        = {Mastering Atari Games with Limited Data},
  booktitle    = {Advances in Neural Information Processing Systems 34: Annual Conference
                  on Neural Information Processing Systems 2021, NeurIPS 2021, December
                  6-14, 2021,},
  year         = {2021},
}

@article{mn15,
author = {Mnih, Volodymyr and Kavukcuoglu, Koray and Silver, David and Rusu, Andrei A and Veness, Joel and Bellemare, arc G and Graves, Alex and Riedmiller, Martin and Fidjeland, Andreas and Ostrovski, Georg and Petersen, Stig and Beattie, Charles and Sadik, Amir and Antonoglou and King, Helen and Kumaran, Dharshan and Wierstra, Daan and Legg, Shane and Hassabis, Demis  },
journal = {Nature},
pages = {529–533},
title = {Human-level control through deep reinforcement learning},
volume = {518},
year = {2015}
}

@article{wang16,
author = { Wang, Ziyu and Schaul, Tom and Hessel, Matteo and Van Hasselt, Hado and Lanctot, Marc and De Freitas, Nando.},
journal = {Internation Conference on Machine Learning ICML.},
title = {Dueling network architectures for deep reinforcement learning.},
year = {2016},
pages={1995–2003},
}

@inproceedings{bell17,
  author    = {Marc G. Bellemare and
               Will Dabney and
               R{\'{e}}mi Munos},
  title     = {A Distributional Perspective on Reinforcement Learning},
  booktitle = {Proceedings of the 34th International Conference on Machine Learning,
               {ICML} 2017, Sydney, NSW, Australia, 6-11 August 2017},
  series    = {Proceedings of Machine Learning Research},
  volume    = {70},
  pages     = {449--458},
  publisher = {{PMLR}},
  year      = {2017},
}

\end{document}